\documentclass[11pt]{article}

\usepackage[T1]{fontenc}

\usepackage[letterpaper, left=1in, right=1in, top=1in, bottom=1in]{geometry}

\usepackage{parskip}

\usepackage[dvipsnames]{xcolor}
\usepackage[colorlinks=true, linkcolor=blue, citecolor=blue, pagebackref]{hyperref}
\usepackage{microtype}

\usepackage{algorithm}
\usepackage{algpseudocode}

\usepackage{natbib}
\bibpunct{(}{)}{;}{a}{,}{,}

\usepackage{amsthm}
\usepackage{mathtools}
\usepackage{amsmath}
\usepackage{bbm}
\usepackage{amsfonts}
\usepackage{amssymb}

\usepackage{MnSymbol} %Actually conflicts with amssymb and others

\theoremstyle{definition}  %Sets style of subsequent newtheorems to 'definition'

\newtheorem{lemma}{Lemma}

\newtheorem{corollary}{Corollary}
\newtheorem{proposition}{Proposition}

\newtheorem{assumption}{Assumption}

\newtheorem{definition}{Definition}
\theoremstyle{plain}

\newtheorem{example}{Example}
\newtheorem{theorem}{Theorem}

\usepackage{xpatch}

\xpatchcmd{\proof}{\itshape}{\normalfont\proofnameformat}{}{}
\newcommand{\proofnameformat}{\bfseries}

\usepackage{prettyref}
\newcommand{\pref}[1]{\prettyref{#1}}
\newcommand{\pfref}[1]{Proof of \prettyref{#1}}
\newcommand{\savehyperref}[2]{\texorpdfstring{\hyperref[#1]{#2}}{#2}}
\newrefformat{eq}{\savehyperref{#1}{\textup{(\ref*{#1})}}}
\newrefformat{eqn}{\savehyperref{#1}{Equation~\ref*{#1}}}
\newrefformat{lem}{\savehyperref{#1}{Lemma~\ref*{#1}}}
\newrefformat{fact}{\savehyperref{#1}{Fact~\ref*{#1}}}
\newrefformat{def}{\savehyperref{#1}{Definition~\ref*{#1}}}
\newrefformat{thm}{\savehyperref{#1}{Theorem~\ref*{#1}}}
\newrefformat{corr}{\savehyperref{#1}{Corollary~\ref*{#1}}}
\newrefformat{sec}{\savehyperref{#1}{Section~\ref*{#1}}}
\newrefformat{app}{\savehyperref{#1}{Appendix~\ref*{#1}}}
\newrefformat{ass}{\savehyperref{#1}{Assumption~\ref*{#1}}}
\newrefformat{ex}{\savehyperref{#1}{Example~\ref*{#1}}}
\newrefformat{fig}{\savehyperref{#1}{Figure~\ref*{#1}}}
\newrefformat{alg}{\savehyperref{#1}{Algorithm~\ref*{#1}}}
\newrefformat{rem}{\savehyperref{#1}{Remark~\ref*{#1}}}
\newrefformat{conj}{\savehyperref{#1}{Conjecture~\ref*{#1}}}
\newrefformat{prop}{\savehyperref{#1}{Proposition~\ref*{#1}}}
\newrefformat{proto}{\savehyperref{#1}{Protocol~\ref*{#1}}}
\newrefformat{prob}{\savehyperref{#1}{Problem~\ref*{#1}}}
\newrefformat{claim}{\savehyperref{#1}{Claim~\ref*{#1}}}

\DeclarePairedDelimiter{\abs}{\lvert}{\rvert} %
\DeclarePairedDelimiter{\brk}{[}{]}
\DeclarePairedDelimiter{\crl}{\{}{\}}
\DeclarePairedDelimiter{\prn}{(}{)}
\DeclarePairedDelimiter{\nrm}{\|}{\|}
\DeclarePairedDelimiter{\tri}{\langle}{\rangle}

\DeclarePairedDelimiter{\ceil}{\lceil}{\rceil}

\DeclareMathOperator{\En}{\mathbb{E}}

\DeclareMathOperator*{\argmin}{arg\,min} % * Places subscript directly under operator
\DeclareMathOperator*{\argmax}{arg\,max}

\newcommand{\ls}{\ell}
\newcommand{\pmo}{\crl*{\pm{}1}}
\newcommand{\eps}{\epsilon}
\newcommand{\veps}{\varepsilon}

\newcommand{\ldef}{\vcentcolon=}

\newcommand{\wt}[1]{\widetilde{#1}}
\newcommand{\wh}[1]{\widehat{#1}}

\def\ddefloop#1{\ifx\ddefloop#1\else\ddef{#1}\expandafter\ddefloop\fi}
\def\ddef#1{\expandafter\def\csname bb#1\endcsname{\ensuremath{\mathbb{#1}}}}
\ddefloop ABCDEFGHIJKLMNOPQRSTUVWXYZ\ddefloop
\def\ddefloop#1{\ifx\ddefloop#1\else\ddef{#1}\expandafter\ddefloop\fi}
\def\ddef#1{\expandafter\def\csname b#1\endcsname{\ensuremath{\mathbf{#1}}}}
\ddefloop ABCDEFGHIJKLMNOPQRSTUVWXYZ\ddefloop
\def\ddef#1{\expandafter\def\csname c#1\endcsname{\ensuremath{\mathcal{#1}}}}
\ddefloop ABCDEFGHIJKLMNOPQRSTUVWXYZ\ddefloop
\def\ddef#1{\expandafter\def\csname h#1\endcsname{\ensuremath{\widehat{#1}}}}
\ddefloop ABCDEFGHIJKLMNOPQRSTUVWXYZ\ddefloop
\def\ddef#1{\expandafter\def\csname hc#1\endcsname{\ensuremath{\widehat{\mathcal{#1}}}}}
\ddefloop ABCDEFGHIJKLMNOPQRSTUVWXYZ\ddefloop
\def\ddef#1{\expandafter\def\csname t#1\endcsname{\ensuremath{\widetilde{#1}}}}
\ddefloop ABCDEFGHIJKLMNOPQRSTUVWXYZ\ddefloop
\def\ddef#1{\expandafter\def\csname tc#1\endcsname{\ensuremath{\widetilde{\mathcal{#1}}}}}
\ddefloop ABCDEFGHIJKLMNOPQRSTUVWXYZ\ddefloop

\newcommand{\burk}{\mathbf{U}}

\let\wt\undefined
\newcommand{\wt}[1]{\widetilde{#1}}
\newcommand{\mb}[1]{\boldsymbol{#1}}

\newcommand{\F}{\mathcal{F}}
	
\newcommand{\X}{\mathcal{X}}
\newcommand{\Y}{\mathcal{Y}}
\newcommand{\T}{\mathcal{T}}

\newcommand{\sB}{{\mathsf B}}

\newcommand{\bp}{\boldsymbol{p}}
\newcommand{\bx}{\boldsymbol{x}}

\newcommand{\bz}{\boldsymbol{z}}

\newcommand{\suff}{\mathbf{T}}

\newcommand{\xr}[1][n]{x_{1:#1}}

\newcommand{\y}{\mathbf{y}}

\newcommand{\pred}{{\widehat{y}}}
\newcommand{\yh}{\pred}

\newcommand{\loss}{{\ell}}
\newcommand{\inner}[1]{\left\langle #1 \right\rangle}
\newcommand{\reals}{{\mathbb R}}

\newcommand{\norm}[1]{\left\|#1\right\|}

\newcommand{\grad}{\nabla}

\newcommand{\trn}{\intercal}

\newcommand{\x}{\mathbf{x}}

\renewcommand{\trn}{\dagger}
\newcommand{\Tr}{\mathsf{tr}}

\newcommand{\multiminimax}[1]{\ensuremath{\left\llangle #1\right\rrangle}}

\newcommand{\reg}{\mathrm{Reg}_n}
\renewcommand{\trn}{\top}
\newcommand{\dl}{\delta}
\newcommand{\sym}{\mathbb{S}}
\newcommand{\Bfun}{Burkholder }

\newcommand{\predt}{\wt{y}}

\newcommand{\propone}{$1^{o}$}
\newcommand{\proptwo}{$2^{o}$}
\newcommand{\propthree}{$3^{o}$}
\newcommand{\propthreep}{$3'$}

\title{\textbf{Online Learning: \\Sufficient Statistics and the Burkholder Method}}

\date{}

\author{
  Dylan J. Foster\thanks{Cornell University}
  \and
  Alexander Rakhlin\thanks{MIT}
  \and
  Karthik Sridharan\footnotemark[1]
  }

\begin{document}

\maketitle

\begin{abstract}
We uncover a fairly general principle in online learning: If regret can be (approximately) expressed as a function of certain ``sufficient statistics'' for the data sequence, then there exists a special \emph{Burkholder function} that 1) can be used algorithmically to achieve the regret bound and 2) only depends on these sufficient statistics, not the entire data sequence,  so that the online strategy is only required to keep the sufficient statistics in memory. This characterization is achieved by bringing the full power of the \emph{Burkholder Method}---originally developed for certifying probabilistic martingale inequalities---to bear on the online learning setting.

To demonstrate the scope and effectiveness of the Burkholder method, we develop a novel online strategy for matrix prediction that attains a regret bound corresponding to the variance term in matrix concentration inequalities. We also present a linear-time/space prediction strategy for parameter free supervised learning with linear classes and general smooth norms. 

\end{abstract}

\section{Introduction}
% !TEX root = paper.tex

Two of the most appealing features of online learning methods are
	 (a) robustness, due to the absence of assumptions on the data-generating process, and (b) the ability to efficiently incorporate data on the fly. According to this latter desideratum, online methods should not store all the data observed so far in memory, but instead maintain some ``compressed'' representation, sufficient for making online predictions. The focus of this work is the study of such \emph{sufficient statistics} for online learning, and the design of computationally efficient methods that employ them.
	
It is natural to turn to Statistics for inspiration: a classical notion of \emph{sufficient statistics} \citep{ra1922mathematical} ensures that a statistician can search for methods that work on ``compressed'' representations of the data. Sufficient statistics have also been studied in sequential decision theory \citep{bahadur1954sufficiency}. However, the very notion of sufficiency is inherently tied to the posited probabilistic model, and the corresponding notion for arbitrary sequences---as postulated by the above desideratum (a)---is all but obvious.

The current theory of online learning offers little guidance as to what summaries of past data should be recorded by an online algorithm. For instance, the Exponential Weights algorithm \citep{vovk1990aggregating,littlestone1994weighted} keeps in memory the cumulative losses of the experts, while the general potential-based forecaster \citep{PLG} updates the cumulative regret of the algorithm with respect to each expert. The methods from the Follow-the-Regularized-Leader family (also known as Dual Averaging methods) work with the sum of gradients of convex functions, while the Online Newton Step \citep{hazan2007logarithmic} method and the Vovk-Azoury-Warmuth forecaster \citep{PLG} also store the ``covariance'' matrix of outer products. The well-known adaptive gradient descent procedure (e.g. \citep{rakhlin2015equivalence}) tunes the step size of online gradient descent according to the cumulative squared norms of gradients, a statistic that appears to be necessary for achieving the adaptive bound, while the ZigZag method of \cite{foster2017zigzag} keeps track of a sign-transformed sequence of the gradients to achieve the empirical Rademacher complexity as a regret bound.

The question of sufficient statistics for online methods appears to be unexplored and poorly understood, and it will take significant effort to answer it. In this paper we propose an approach that appears to be general yet, inevitably, incomplete. We propose a definition that brings many existing methods under the same umbrella, and allows us to develop new efficient strategies that have been out of reach. The key workhorse for our development is the Burkholder method, studied in probability theory and harmonic analysis.

Beyond studying a notion of sufficient statistic for online methods, our work can be seen as providing a further understanding of emerging connections between online learning, martingale inequalities, and deterministic geometric quantities. At the risk of being imprecise, let us describe the bird's-eye view of our overall approach:
\begin{figure}[H]
	\label{fig:high_level_equiv}
  \centering
    \includegraphics[width=\textwidth]{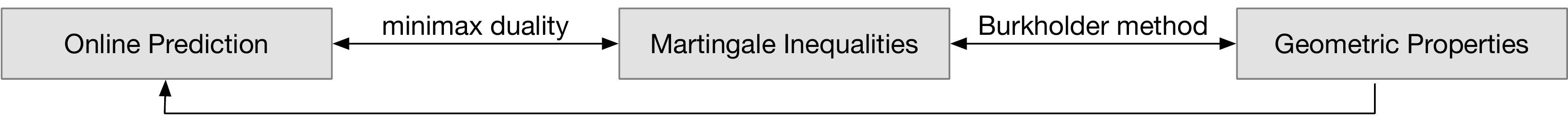}
	\vspace{-9mm}
\end{figure}
Based on the definition of sufficient statistics for online methods, we first derive the corresponding martingale inequalities with the help of the minimax theorem. We then turn to the Burkholder method, and show equivalence of these martingale inequalities with sufficient statistics and existence of a special Burkholder (or Bellman) function, a purely geometric object. We then use this function for the problem of online prediction, thus completing the circle. Crucially, the sufficient statistics we start with are reflected in the Burkholder function, and, hence, the proposed algorithm is only required to update these compressed representations of the data. We exhibit the power of this approach by deriving several new efficient prediction methods.

We remark that \citep{foster2017zigzag} studied a particular case of the Burkholder method related to the UMD property. The present work shows that the approach can be generalized significantly and used to address the question of sufficient statistics. For example, the explicit construction of the UMD-style Burkholder function for certain matrix prediction problems was noted to be challenging in \citep{foster2017zigzag} and indeed does not appear to be known in the analysis community \citep{osekowski2017personal}. In spite of this, the approach in the present paper uses different sufficient statistics to attain the same results with an explicit (and efficient) Burkholder function.

\section{Problem Setup and Sufficient Statistics}
\label{sec:problem}

Consider the \emph{Online Supervised Learning} setting where, for each round $t=1,\ldots,n$, the forecaster observes side information $x_t\in\X$, makes a prediction $\pred_t\in\Y\subset \reals$, observes the outcome $y_t\in\Y$, and incurs the loss of $\loss(\pred_t,y_t)$, where $\ls:\bbR\times{}\bbR\to\bbR$. In a general form, the goal of the forecaster is to ensure that
\begin{align}
	\label{eq:def_phi_regret}
	\En\left[\sum_{t=1}^n \loss(\pred_t,y_t)\right] \leq \phi(x_1,y_1,\ldots,x_n,y_n)
\end{align}
for any sequence $(x_1,y_1),\ldots,(x_n,y_n)$, where the expectation is with respect to forecaster's randomization. The choice of $\phi$ models the problem at hand, and examples in this paper focus on
\begin{align}
	\label{eq:phi_comp_adap}
	\phi(x_1,y_1,\ldots,x_n,y_n) = \min_{f\in\F} \left\{ \sum_{t=1}^n \loss(f(x_t),y_t) + \cA(f, x_1,\ldots,x_n)\right\},
\end{align}
for some class of functions $\F:\X\to\reals$ and an \emph{
adaptive bound} $\cA:\F\times\X^n\to\reals$. In this case, the difference of the cumulative losses of the forecaster and an $f\in\F$ is commonly referred to as \emph{regret},
$$\reg(f) = \sum_{t=1}^n\loss(\pred_t,y_t)-\loss(f(x_t),y_t).$$
We assume that $\phi$ is uniformly bounded over $(\cX\times{}\cY)^{n}$. We further assume that $\loss$ is convex and $L$-Lipschitz in the first argument over $\cY$. We denote the derivative (or a subderivative) of $\loss(\cdot,y)$ at $\pred$ by $\partial \loss(\pred,y)\in[-L,L]$. We will abbreviate $\delta_t = \partial\loss(\pred_t,y_t)$ when it is clear from context, but keep in mind that this value depends on the two variables $\yh_t$ and $y_t$. 
We assume that for any distribution $p$ on $\Y$, $\arg\min_{\pred\in\bbR} \En_{y \sim p}\loss(\pred,y) \in \cY$, and that $\cY$ is compact. We let $\Delta_{\cY}$ denote the space of all Borel probability measures on $\cY$ (more generally, $\Delta_{A}$ will denote the set of Borel probability measures over some set $A$). Since $\cY$ is compact, Prokhorov's theorem implies that $\Delta_{\cY}$ is compact in the weak topology. This enables application of the minimax theorem as in previous works in this direction \citep{RakSriTew10,RakSriTew14jmlr,FosRakSri15}.

\textbf{Additional notation}~~~~
 Given a function $f:S\to\bbR$, its Fenchel dual $f^{\star}$ is defined via $f^{\star}(w) = \sup_{x\in{}S}\crl*{\tri*{w,x}-f(x)}$. For any norm $\nrm*{\cdot}$, the dual norm will be denoted by $\nrm*{\cdot}_{\star}$. $\sB_p^d$ will denote the $d$-dimension unit $\ls_p$ ball and the shorthand $\Delta_{d}$ will denote the simplex in $d$ dimensions. For any interval $\brk*{a,b}$, we let $\mathrm{proj}_{\brk*{a,b}}(x) = \min\crl*{b, \max\crl*{a, x}}$.

\subsection{Sufficient Statistics}

Since there is no probabilistic model for data in the online learning setting, the notion of ``sufficiency'' has to be tied to the particular choice of $\phi$. It is then tempting to define a sufficient statistic as a ``compressed'' representation which may be used by some strategy to ensure \pref{eq:def_phi_regret}. While natural, such a definition does not provide any additional structure to narrow the search for an algorithm.

The definition we propose is as follows:
\begin{definition}
	\label{def:sufficiency}
	Let $\T$ be some vector space. A function $\suff: \X \times \Y \times [-L,L] \to \T$ is an \emph{additive sufficient statistic} for $\phi$ if there exists $V: \T \to \reals$ such that
\begin{align}
	\label{eq:suffiency_def}
\sum_{t=1}^n \loss(\pred_t,y_t)- \phi(x_1,y_1,\ldots,x_n,y_n) \le  V\left(\sum_{t=1}^n \suff(x_t,\pred_t,\partial \loss(\pred_t,y_t))\right) 
\end{align}
for any sequence $x_1,\pred_1,y_1, \ldots, x_n,\pred_n,y_n$. We refer to $(\suff, V)$ as a \emph{sufficient statistic pair}.
\end{definition}
In \pref{sec:discussion}, we consider a more general non-additive definition. All examples in this paper, however, are already covered by \pref{def:sufficiency}, and we will drop the word ``additive'' for now. We will also make the mild assumption that there exists $(x^0,y^0)\in\X\times\Y$ such that $\suff(x^0,y^0,0)=0\in \T$.
\begin{example}[Prediction with expert advice]
	\label{ex:experts}
	Consider $\phi$ as in Eq. \pref{eq:phi_comp_adap} with $\cF$ as the set of linear functions $f(x)=\inner{f,x}$ for $f\in\Delta_d$, with $\X=[-1,1]^d$, and with non-adaptive rate $\cA \ldef c\sqrt{n\log d}$. Then the left-hand-side of \pref{eq:suffiency_def} can be upper bounded via linearization of the convex loss by 
	$$\max_{j \in 1,\ldots,d}\sum_{t=1}^n \partial\loss(\pred_t,y_t) \cdot (\pred_t-\inner{e_j,x_t}) - c\sqrt{n\log d}.$$
	It follows that $\reals^d$-valued map $\suff$ defined by $[\suff(x_t,\pred_t,\delta_t)]_j = \delta_t\cdot (\pred_t-\inner{e_j,x_t})$ is a sufficient statistic.
\end{example}
\begin{example}[Adaptive Gradient Descent]
	Consider $\phi$ as in Eq. \pref{eq:phi_comp_adap} with $\cF$ as the set of linear functions $f(x)=\inner{f,x}$ for $f\in \sB_2^d$, $\X=\reals^d$, and adaptive bound $\cA(\nabla_{1},\ldots\nabla_n) \ldef (\sum_{t=1}^n \norm{\nabla_t}^2)^{1/2}$, where $\nabla_t \ldef \delta_t x_t$. The left-hand-side of \pref{eq:suffiency_def} is at most
	$$\max_{f\in \sB_2^d}\sum_{t=1}^n \delta_t \cdot (\pred_t-\inner{f,x_t}) - (\sum_{t=1}^n \norm{\nabla_t}^2)^{1/2} = \sum_{t=1}^n \delta_t \cdot \pred_t+\norm{\sum_{t=1}^n \nabla_t} - (\sum_{t=1}^n \norm{\nabla_t}^2)^{1/2},$$
	This implies that $\suff(x_t,\pred_t,\delta_t) = \left( \delta_t\pred_t, \nabla_t, \norm{\nabla_t}^2 \right)\in\reals\times \X\times \reals$ is a sufficient statistic.
\end{example}

\section{Martingale Inequalities and the Burkholder Method}
\label{sec:burkholder}

The notion of sufficient statistics introduced in the previous section will only be useful if we exhibit a prediction strategy employing this type of representation. Before doing so, we need to build the two bridges outlined in the diagram on the previous page.
These correspond to \pref{lem:suff_to_martingale} and \pref{lem:equivalence_burkholder} below.

First, we show that existence of a prediction strategy that guarantees the regret inequality \pref{eq:def_phi_regret} for all sequences can be ensured by checking a martingale inequality involving only the sufficient statistics. The key tool in proving the lemma is the minimax theorem. 

Note that in a slight abuse of notation, we will concatenate the first two arguments of any sufficient statistic $\suff$ and write them as $z_{t}\ldef(x_t, \pred_t)$ going forward.

\begin{lemma}
	\label{lem:suff_to_martingale}
Suppose $(\suff,V)$ is a sufficient statistic pair for $\phi$. Let $\delta=(\delta_1,\ldots,\delta_n)$ be a $[-L,L]$-valued martingale difference sequence (i.e. $\En[\delta_t\mid\cG_{t-1}]=0$, where $\cG_{t-1}=\sigma(\delta_1,\ldots,\delta_{t-1})$). Let $\bz=(\bz_1,\ldots,\bz_n)$ be a sequence of functions $\bz_t: [-L,L]^{t-1} \to \X \times \Y$, each viewed as a predictable process with respect to $\cG_{t-1}$. Then a sufficient condition for existence of a prediction strategy such that \pref{eq:def_phi_regret} holds for all sequences $(x_1,y_1),\ldots,(x_n,y_n)$ is that
\begin{align}
	\label{eq:suff_prob_ineq}
	\En\left[  V\left(\sum_{t=1}^n \suff(\bz_t, \delta_t) \right) \right] \le 0
\end{align}
holds for any $\bz$ and any law of $\delta$. Moreover, when $\alpha \mapsto V(\tau+\suff(z,\alpha))$ is convex for any $z\in\X\times\Y,\tau\in\T$, it is enough to check \pref{eq:suff_prob_ineq} for $\delta_t=\epsilon_t \cdot L$, $t=1,\ldots,n$, where $\epsilon_t$s are independent Rademacher random variables.
\end{lemma}	

\pref{lem:suff_to_martingale} is in the spirit of results in \citep{RakSriTew10,RakSriTew14jmlr,FosRakSri15} whereby existence of a strategy (or, ``learnability'') is certified non-constructively by proving a martingale inequality. 

The next lemma provides a key insight into existence of certain deterministic functions with ``geometric'' properties (in particular, \emph{restricted concavity}) and can be seen as a variation on the so-called \emph{Burkholder method}  (also sometimes called the \emph{Bellman function method}; see \citep{osekowski2012sharp} for the detailed treatment and examples).

\begin{lemma}
	\label{lem:equivalence_burkholder}
Let $\delta=(\delta_1,\ldots,\delta_n)$ be a $[-L,L]$-valued martingale difference sequence with joint law $\bp$ and let $\bz=(\bz_1,\ldots,\bz_n)$ be a predictable process ($\bz_t: [-L,L]^{t-1} \to \X \times \Y$) with respect to $\cG_{t-1}=\sigma(\delta_1,\ldots,\delta_{t-1})$. The probabilistic inequality 
\begin{align}
	\label{eq:martingale_nonpositive}
	\En\left[  V\left(\sum_{t=1}^n \suff(\bz_t,\delta_t)\right) \right] \le 0
\end{align}
holds for any $n\geq 1$, $\bz$, and $\bp$ \emph{if and only if} one can find function $\burk:\T\to\reals$ that satisfies the following three properties:
\begin{enumerate}
 \item[$1^o$] $\burk(0) \le 0$.
 \item[$2^o$] For any $\tau \in \T$, $\burk(\tau) \ge V(\tau)$.
 \item[$3^o$] For any $\tau \in \T$, $z \in \X \times \Y$, and any mean-zero distribution $p$ on $[-L,L]$, 
 \begin{align}
 \En_{\alpha\sim p}\left[\burk(\tau + \suff(z,\alpha))\right] \le \burk(\tau). \hspace{-20mm} \tag{restricted concavity}
 \end{align}
\end{enumerate}
Furthermore, if for any $\tau \in \T$ and $z \in \X \times \Y$ the mapping $\alpha \mapsto V(\tau + \suff(z,\alpha))$ is convex, then condition \pref{eq:martingale_nonpositive} can be relaxed with $\delta$ replaced by independent Rademacher random variables $(\epsilon_1,\ldots,\epsilon_n)$. In this case the following property holds for $\burk$:
\begin{enumerate}
 \item[$3'$] The mapping $\alpha \mapsto \burk(\tau + \suff(z,\alpha))$ is convex and (property $3^o$ is replaced by): 
 $$
 \forall \tau \in \T, z \in \X \times \Y,~~~ \En_\epsilon \burk(\tau + \suff(z,\epsilon L))  \le \burk(\tau),
 $$
 where $\epsilon$ is a Rademacher random variable. 
\end{enumerate}
\end{lemma}

\begin{definition}
We call any function $\burk$ satsifying the properties $1^o$, $2^o$, and $3^o$/$3'$ a \emph{Burkholder function for $(\suff,V)$}.
\end{definition}

In plain language, the lemma says that one can prove a certain probabilistic inequality if and only if there is a deterministic function with certain properties. The proof of the lemma, in fact, provides a construction for the ``optimal'' function $\burk$, but it is not clear how to directly evaluate the optimal function efficiently (see \pref{sec:discussion} for a discussion of the computational prospects of automating this process).

We remark that the Burkholder functions guaranteed by the lemma are not unique, and some may be easier to find than others. We also note that any Burkholder function $\burk$ for $(\suff,V)$ yields another sufficient statistic pair $(\suff, \burk)$ guaranteeing the same regret bound. The power of \pref{lem:equivalence_burkholder} is to guarantee the existence of a function $\burk$ satisfying property \propthree{} when the function $V$ under consideration does not have these properties. This situation, where the choice of $V$ is ``obvious'' but the discovery of $\burk$ requires nontrivial analysis, occurs frequently when one attempts to design adaptive algorithms for a new task.

To showcase the power of this lemma, we consider a particular martingale inequality that gives rise to the geometric notions of strong convexity and smoothness. These geometric properties are extensively employed in Online Convex Optimization: to instantiate the Mirror Descent algorithm with a given norm, one needs to exhibit a function that is strongly convex with respect to a given norm of interest; For example, for the $\ls_{1}$ norm a standard choice is the negative entropy function. The next example shows that for any norm, the optimal strongly convex function is precisely the dual of the Burkholder special function for a particular martingale inequality. This example is the focus of \citep{Pisier75}, yet for us it is one point on the spectrum of sufficient statistics.

\begin{example}[Smoothness and Strong Convexity] 
	\label{ex:smoothness}
	Assume $L=1$ for brevity. Suppose $\X=\reals^d$ (more generally, we may take $\cX$ to be a Banach space), equipped with a norm $\norm{\cdot}$. Let $V:\X\times \reals\to\reals$ be defined by $(x,a) \mapsto \norm{x}^2-C \cdot a$ ~ for $C>0$. Take $\suff(x_t,\pred_t,\delta_t)=(\delta_t x_t, \norm{x_t}^2)$. Since $\alpha\mapsto V(\tau+\suff(x_t,\pred_t,\alpha))$ is convex, it is enough to consider \pref{eq:martingale_nonpositive} for independent Rademacher random variables. The martingale inequality \pref{eq:martingale_nonpositive} then reads
\begin{align}
	\label{eq:smoothness_martingale_ineq}
	\En\left[ \norm{\sum_{t=1}^n \epsilon_t\bx_t}^2 - C\sum_{t=1}^n \norm{\bx_t}^2\right] \leq 0
\end{align}
for any $\X$-valued predictable process $(\bx_t)$ with respect to the dyadic filtration $\F_{t-1} = \sigma(\epsilon_1,\ldots,\epsilon_{t-1})$. \pref{lem:equivalence_burkholder} guarantees existence of a Burkholder function $\burk$, and property $3'$ reads 
$$\En_\epsilon \burk(\tau_1 + \epsilon x, \tau_2 + \norm{x}^2)  \le \burk(\tau_1,\tau_2),$$
for any $\tau=(\tau_1,\tau_2)\in\X\times \reals$ and $x\in\X$. From the construction of $\burk$ in the proof of \pref{lem:equivalence_burkholder}, with our particular choice of $V$, one can deduce that $\burk(\tau_1,\tau_2)= \burk(\tau_1,0)+\tau_2$. Hence,
$$\frac{1}{2}\left( \burk(\tau_1 + x, 0)+ \burk(\tau_1-x, 0) \right) + C\norm{x}^2 = \frac{1}{2}\left( \burk(\tau_1 + x, C\norm{x}^2)+ \burk(\tau_1-x, C\norm{x}^2) \right) \leq \burk(\tau_1,0)$$
and, thus, $x\mapsto \burk(x,0)$ is smooth with respect to the norm and its dual is strongly convex with respect to $\nrm*{\cdot}_{\star}$. In summary, the Burkholder method captures the geometry necessary for defining Gradient-Descent-style methods, as the dual of $\burk(x,0)$ provides the universal construction for a strongly convex function with respect to a given norm. See \cite{srebro2011universality} for an in-depth treatment of Mirror Descent and universal construction of strongly convex regularizers.
\end{example}

What should an algorithm designer take away from the developments thus far? Let us provide a brief summary. One first starts with a desired regret inequality for the online learning setting, such as \pref{eq:def_phi_regret}. The next step is to find an upper bound on the regret inequality that can be expressed in terms of additive sufficient statistics. \pref{lem:suff_to_martingale} and \pref{lem:equivalence_burkholder} then guarantee, respectively, that there is a certain martingale inequality that must hold if the upper bound in terms of sufficient statistics is achievable, and that there must exist a Burkholder function with certain geometric properties. In the next section we close the loop by showing that whenever such a Burkholder function can be evaluated efficiently, it yields an efficient algorithm that only keeps the sufficient statistics in memory.

Before proceeding, we briefly remark that the sufficient statistics expansion $V$ also serves as a lower bound on the regret inequality, then there is a formal sense in which the special Burkholder function exists if and only if there exists a strategy achieving the original regret inequality of interest; this is the focus of \pref{sec:necessary}. In the reverse direction, one may start with a probabilistic inequality and determine the statistics that should be used to define the online prediction goal.\footnote{This was precisely the approach used to develop a matrix prediction method we present in \pref{sec:matrix}.}

\section{The Burkholder Algorithm}
\label{sec:algorithm}

\pref{ex:smoothness} in the previous section already suggests that the Burkholder $\burk$ functions may capture ``geometry'' needed for forming online predictions. The example is also suggesting that strong convexity and smoothness may not be sufficient for prediction problems where more complicated sufficient statistics (beyond the norm of the sum and the sum of the squared norms) are necessary. Thankfully, the function $\burk$ reflects any sufficient statistics for the online prediction problem. We can now define a ``universal'' algorithm that has access to $\burk$.

To define the algorithm, first let
$\zeta_{t-1} = \sum_{j=1}^{t-1} \suff(x_j,\pred_j,\delta_j)$
be the cumulative value of the sufficient statistic computed after $t-1$ rounds. Since $\T$ is a vector space, $\zeta_{t}$s are elements of $\T$, and this is the only information the algorithm stores in memory. 

The \emph{Burkholder algorithm} is defined by the update:
\begin{align}
	\label{eq:univ_algo}
\textbf{Compute}\;\; q_t  = \argmin_{q\in\Delta_{\cY}} ~\sup_{y\in\cY}~ \En_{\pred\sim q} \burk\Big(\zeta_{t-1} + \suff(x_t,\pred,\partial \loss(\pred,y))\Big).
\quad\quad\quad\textbf{Sample}\;\;\pred_{t}\sim{}q_t.
\end{align}

\begin{lemma}
  \label{lem:universal_algo}
	For a sufficient statistic pair $(\suff,V)$, if there exists a Burkholder function $\burk$ satisfying Properties $1^o$, $2^o$, and $3^o$ (or $3'$) of \pref{lem:equivalence_burkholder}, then the Burkholder algorithm \pref{eq:univ_algo} obtains the regret bound \pref{eq:def_phi_regret} in expectation for all sequences $(x_1,y_1),\ldots,(x_n,y_n)$. 
\end{lemma}

\begin{proof}
To check that the above strategy works, fix a value $x_t$ and observe that by the minimax theorem,\footnote{The minimax theorem can be applied because $\Delta_{\cY}$ is compact; see discussion in the proof of \pref{lem:suff_to_martingale}.}
\begin{align*}
\inf_{q\in\Delta_{\cY}} \sup_{y\in\cY} & \En_{\pred \sim q\in\Delta{\cY}} \burk\left(\zeta_{t-1} + \suff(x_t,\pred,\partial \loss(\pred,y))\right) = \sup_{p\in\Delta_{\cY}} \inf_{\pred\in\cY}  \En_{y \sim p} \burk\left(\zeta_{t-1} + \suff(x_t,\pred,\partial \loss(\pred,y))\right)
\end{align*}
For any fixed $p$, let $\pred^{\star} \ldef \argmin_{\pred\in\cY} \En_{y \sim p} \loss(\pred,y)$, which implies $\partial \loss(\pred^{\star},y)$ is a zero mean variable (see the proof of \pref{lem:suff_to_martingale}). Taking the worst case value for $p$ and choosing $\yh^{\star}$ as the learner's strategy for each $p$ yields an upper bound of $\sup_{p\in\Delta_{\cY}} \En_{y \sim p} \burk\left(\zeta_{t-1} + \suff(x_t,\pred^{\star},\partial \loss(\pred^{\star},y))\right)$,
which in turn is upper bounded by
\begin{align*}
\sup_{\pred^{\star}\in\cY}\sup_{p\in\Delta_{\brk*{-L,L}} \;:\; \En_{\alpha \sim p}[\alpha]=0}   \En_{\alpha \sim p} \burk\left(\zeta_{t-1} + \suff(x_t,\pred^{\star},\alpha)\right)
\end{align*}
by observing that the distribution over $\partial \loss(\pred^{\star},y)$ belongs to the set of all zero-mean distributions supported on $[-L,L]$. The third property of $\burk$ now leads to the upper bound,
\[
\sup_{\pred^{\star}\in\cY}\sup_{p\in\Delta_{\brk*{-L,L}} \;:\; \En_{\alpha \sim p}[\alpha]=0}   \En_{\alpha \sim p} \burk\left(\zeta_{t-1} + \suff(x_t,\pred^{\star},\alpha)\right) \leq{} \burk\left(\zeta_{t-1}\right).
\]
Applying this argument from $t=n$ down to $t=0$ yields the value $\burk(0)\leq 0$.
\end{proof}

We remark that the approach presented here extends beyond the relaxation framework of \citep{rakhlin2012relax}. In particular, the present approach can handle recursions which cannot be written in the form ``$\loss(\pred_t,y_t)+\text{Rel}(x_{1:t},y_{1:t})$'', e.g. when the potential function has to depend on past forecasts $(\pred_t)$.

\textbf{Implementation}~~~~ When $\burk$ is convex in $\pred$ and the set $\Y$ is convex, the minimum over $q$ is achieved at a deterministic strategy, and so the  minimization problem simplifies to $\argmin_{\pred\in\Y}$. All of the Burkholder functions we explore in this paper enjoy this or similar simplified and efficient representations for the algorithm. These simplifications are detailed in \pref{app:efficient}. Even without convexity, the general form for the Burkholder algorithm in \pref{eq:univ_algo} can be implemented efficiently via convex programming assuming only Lipschitz continuity of $\burk$.
\begin{proposition}
Suppose $\burk$ is Lipschitz and bounded and can be evaluated in constant time. Then \pref{eq:univ_algo} can be implemented approximately so as to achieve the regret inequality \pref{eq:def_phi_regret} up to additive constants in time $\textrm{poly}(n)$.
\end{proposition}
See \pref{prop:burkholder_efficient} in the appendix for a precise version of this statement.

%%% Local Variables:
%%% mode: latex
%%% TeX-master: "paper"
%%% End:

\section{Example: Fast and Easy Parameter-Free Online Learning}
\label{sec:linear_loss}
% !TEX root = paper.tex

So far all of our examples have concerned adaptive bounds $\cA$ that adapt to the data sequence $x_{1},\ldots,x_{n}$, not the comparator $f$. In this section we will show that the framework of Burkholder functions and sufficient statistics readily encompasses comparator-dependent norms by giving a new family of algorithms for the problem of \emph{parameter-free online learning} \citep{mcmahan2014unconstrained}. The setup is as follows: We equip the subset $\cX\subseteq{}\bbR^{d}$ with a norm $\nrm*{\cdot}$ and assume that $\nrm*{x_{t}}\leq{}1$ for all $t$.\footnote{The result we will present easily extends to the general Banach space case; this is only to simplify presentation.} Recall that $\norm{\cdot}_{\star}$ will denote the dual norm. Rather than constraining the benchmark class to a compact set, we set $\cW=\bbR^{d}$ and set $\cF=\crl*{x\mapsto\tri*{w,x}\mid{}w\in\cW}$. We assume smoothness of the norm: letting $\Psi(x) = \frac{1}{2}\nrm*{x}^{2}$, it holds that\footnote{Our analysis extends to the general case where we instead have  $\frac{1}{2}\nrm*{x}^{2}\leq{}\Psi(x)$ for some $\Psi\neq{}\frac{1}{2}\nrm*{\cdot}^{2}$ and the same smoothness inequality holds, which is needed for settings such as $\ls_1$/$\ls_{\infty}$.}
  \[
    \Psi(x+y) \leq{} \Psi(x) + \tri*{\grad{}\Psi(x), y} + \frac{\beta}{2}\nrm*{y}^{2}.
  \]

To ease notational burden, we will assume the loss is $1$-Lipschitz in this section. We will efficiently obtain a regret bound of the form
\begin{equation}
  \label{eq:regret_param}
\reg(w) \leq{} \cA(w) \ldef \nrm*{w}_{\star}\sqrt{2\beta{}n\log\prn*{\sqrt{\beta}n\nrm*{w}_{\star} + 1}} + 1 \quad\forall{}w\in\bbR^{d}
\end{equation}
for any such smooth norm. We begin by stating a sufficient statistic representation for the problem. This is based on a familiar potential which has appeared in previous works on parameter-free online learning (e.g. \citep{mcmahan2014unconstrained}) in Hilbert spaces; we extend it to any smooth norm, then use it in the Burkholder method to provide \emph{the first linear time/linear space algorithm for parameter-free learning with general smooth norms in online supervised learning}.\footnote{Since the original submission of this paper, the independent work of \citep{cutkosky2018blackbox} has provided an algorithm with a similar regret guarantee and computational efficiency.}

\begin{proposition}
  \label{prop:param_sufficient} Suppose we are interested in an adaptive regret bound of
\[
\cA(w) = \nrm*{w}_{\star}\sqrt{2an\log\prn*{\frac{\sqrt{an}\nrm*{w}_{\star}}{\gamma} + 1}} + c
\]
for constants $a,\gamma,c>0$. Then $\suff(x_t,\pred_t,\delta_t) = \left( \delta_t\cdot\pred_t, \delta_t\cdot x_t\right)\in\reals\times \X$ and the function
\begin{equation}
\label{eq:param_v}
V(b, x) = b + \gamma\exp\prn*{\frac{\nrm*{x}^{2}}{2an}} - c,
\end{equation}
yield a sufficient statistic pair for the regret bound $\cA$.

\end{proposition}

Because the regret bound we provide is not horizon independent unlike previous examples, it will be convenient to allow time-indexed Burkholder functions $(\burk_{t})_{t\leq{}n}$. This indexing is of purely notational convenience, as time-dependent Burkholder functions fit squarely into the algorithmic framework of \pref{lem:universal_algo} by enlarging $\cX$ to $\cX\times{}\brk*{n}$. Nonetheless, we recap the analogous properties for time-dependent Burkholder functions in the proof of the following theorem.

\begin{theorem}
  \label{thm:param_free}
  Suppose $c=1$, $a=\beta$, and $\gamma=1/\sqrt{n}$ in \pref{eq:param_v}. Then
  \[
    \burk_{t}(b, x) \ldef b + \frac{1}{\sqrt{n}}\exp\prn*{\frac{\nrm*{x}^{2}}{2\beta{}t} + \frac{1}{2}\sum_{s=t+1}^{n}\frac{1}{s}} - 1, 
  \]
  is a family of time-varying Burkholder functions satisfying $1^o$, $2^o$, and $3'$.
\end{theorem}

This Burkholder function immediately yields both a prediction strategy achieving \pref{eq:regret_param} and a simple probabilistic martingale inequality. We will now state them both. Because $(\burk_t)_{t\leq{}n}$ satisfy additional convexity properties, the strategy is especially efficient (per \pref{app:efficient} and \pref{lem:det_strat3}).

\begin{corollary}
Suppose that $\cY=\brk*{-B, B}$ for some $B>0$. Then the deterministic prediction strategy
  \[
    \yh_{t} = \mathrm{proj}_{\brk*{-B, B}}\prn*{-\frac{1}{\sqrt{n}}\En_{\sigma\in\pmo}\brk*{\sigma\cdot\exp\prn*{\frac{\nrm*{\sum_{s=1}^{t-1}\dl_{s}x_s + \sigma{}x_t}^{2}}{2\beta{}t} + \frac{1}{2}\sum_{s=t+1}^{n}\frac{1}{s}}}}
  \]
  leads to a regret bound of
  \[
    \sum_{t=1}^{n}\ls(\yh_t, y_t) - \sum_{t=1}^{n}\ls(\tri*{w,x_t}, y_t) \leq{} \nrm*{w}_{\star}\sqrt{2\beta{}n\log\prn*{\sqrt{\beta}n\nrm*{w}_{\star} + 1}} + 1 \quad\forall{}w\in\bbR^{d}.
  \]
\end{corollary}

The Burkholder function family stated above and \pref{lem:equivalence_burkholder} certify that $\sup\En\brk*{V}\leq{}0$. One special case of this martingale inequality is the following mgf bound for vector-valued martingales under smooth norms.

\begin{corollary}
  Let $\bx_{t}(\eps) \ldef \bx_{t}(\eps_1,\ldots,\eps_{t-1})$ be adapted to the filtration $\cF_{t-1}=\sigma(\eps_{1},\ldots,\eps_{t-1})$ for Rademacher random variables $\eps_{1},\ldots,\eps_{n}$, and let $\nrm*{\bx_t}\leq{}1$ almost surely, where $\nrm*{\cdot}$ is a $\beta$-smooth norm. Then it holds that
  \[
    \En_{\eps}\exp\prn*{\frac{\nrm*{\sum_{t=1}^{n}\eps_{t}\bx_{t}(\eps)}^{2}}{2\beta{}n}} \leq{} \sqrt{n}.
  \]

\end{corollary}

\noindent\textbf{Related work~~} Parameter-free online learning is a very active area of research, but essentially all results in this area that we are aware of \citep{mcmahan2013minimax,mcmahan2014unconstrained,orabona2014simultaneous,orabona2016coin,cutkosky2016online, cutkosky2017online} only provide regret bounds of the form \pref{eq:regret_param} in the special case where $\nrm*{\cdot}$ is a Hilbert space. The only exception is \citep{foster2017parameter} which gives an algorithm for smooth norms $\nrm*{\cdot}$, but has time $\mathrm{poly}(n)$ per step.
Our Burkholder-based algorithm has running time $O(d)$ per step and only $O(d)$ memory.\footnote{Technically our algorithm only applies to the online supervised learning setting, whereas the algorithm of \cite{foster2017parameter} applies to the OCO setting.} The key ingredient to achieving this improvement was to examine a known potential through the lens of the Burkholder method. We hope that this approach can lead to similarly useful improvements by applying the Burkholder method to construct more sophisticated potentials as in, e.g. \citep{orabona2016coin, cutkosky2017online}, particularly to achieve regret bounds that adapt jointly to the model and to data.

%%% Local Variables:
%%% mode: latex
%%% TeX-master: "paper"
%%% End:

\section{Example: Matrix Prediction}
\label{sec:matrix}
% !TEX root = paper.tex

In this section we focus on linear matrix prediction problems. The side information $x_t$ is now matrix-valued, and we shall denote it by a capital letter $X_t\in\bbR^{d_1\times{}d_2}$. Our goal is to achieve a regret inequality as in \pref{eq:phi_comp_adap} with a class $\F=\crl*{X\mapsto\tri*{W,X}\mid{} W\in\cW}$, where $\cW=\crl*{W\in\bbR^{d_1\times{}d_2}\mid{}\nrm*{W}_{\Sigma}\leq{}r}$. Here $\tri*{A,B}=\Tr(AB^{\trn})$ is the standard matrix inner product and $\nrm*{\cdot}_{\Sigma}$ denotes the nuclear norm. We also let $\nrm*{\cdot}_{\sigma}$ denote the spectral norm. As before, the loss $\loss$ is assumed to be $L$-Lipschitz and regret against a matrix $W\in\cW$ is given by
$
  \reg(W)\ldef \sum_{t=1}^{n}\loss(\pred_t, y_t) - \loss(\tri*{W,X_t}, y_t).
$

In a search for an adaptive bound on regret, we inspect the adaptive bound for the vector case. The direct analogue for matrices would be a bound proportional to $\left(\sum_{t=1}^n \norm{X_t}^2_\sigma \right)^{1/2}$, and indeed such a bound is possible with Matrix Exponential Weights \cite[Theorem 13]{HazKalSha12}.\footnote{With more work it is possible to obtain a bound of $\left(\max_{t}\nrm*{X_t}_{\sigma}\cdot\nrm*{\sum_{t=1}^nX_t}_\sigma \right)^{1/2}$; this is still weaker than our result, and seems to only be possible when the constraint set and $X_t$s are restricted to be positive-semidefinite.} However, matrix version of Khintchine inequality, as well as matrix deviation inequalities, involve---for the case of random centered self-adjoint matrices---the tighter quantity $\norm{\sum_{t=1}^n X_t^2}^{1/2}_\sigma$ (see \citep{tropp2012user,mackey2014matrix}). Given the correspondence between online regret bounds and martingale inequalities, one may wonder if there is an algorithm that achieves this adaptive bound. We shall exhibit such a method using our approach, and the reader can already guess that $\sum_{t=1}^n X_t^2$ should be part of the sufficient statistic for the online algorithm. We present results for general non-square matrices.

Let $\sym^{d}$ denote the set of symmetric matrices in $\bbR^{d\times{}d}$, $\sym^{d}_{+}$ denote the set of positive-semidefinite matrices, and $\sym^{d}_{++}$ denote the set of positive-definite matrices. 
For $X\in\sym^{d}$ we let $\lambda(X)\in\bbR^{d}$ denote its eigenvalues arranged in decreasing order, so that $\lambda_1(X)$ is the largest eigenvalue. For any matrix $X\in\bbR^{d_1\times{}d_2}$ we define its Hermitian dilation $\cH(X)\in\sym^{d_1+d_2}$ and  $\cM(X) \in\sym^{d_1+d_2}$ as:
\begin{equation}
  \cH(X) = \left(
    \begin{array}{ll}
      0 & X\\
      X^{\trn} & 0
    \end{array}
    \right) ~~~~~~~~~\cM(X) = \cH(X)^{2} = \left(
    \begin{array}{ll}
      XX^{\trn} & 0\\
      0 & X^{\trn}X
    \end{array}
  \right).
  \end{equation}
  It is well-known that for any matrix $X$, $\lambda_{1}(\cH(X)) = \nrm*{X}_{\sigma}.$

With these definitions in place, the desired adaptive regret bound takes the form
\begin{equation}
\cA_{\eta}(X_{1},\ldots,X_{n}) = \frac{\eta{}rL^2}{2}\nrm*{\sum_{t=1}^{n}\cM(X_t)}_{\sigma} + \frac{c}{\eta}
\end{equation}
for some fixed $\eta>0$ and constant $c>0$. The sufficient statistic takes values in  $\cT=\reals\times \sym^{d_1+d_2}\times \sym^{d_1+d_2}_{+}$ and incorporates the matrix variance terms $\cM(X_t)$.

\begin{proposition}
  \label{prop:matrix_sufficient}
  The pair $(\suff, V)$ defined via $\suff(X_t,\pred_t,\delta_t) = \left( \delta_t\cdot\pred_t, \delta_t\cdot \cH(X_t), \cM(X_t) \right)\in\reals\times \sym^{d_1+d_2}\times \sym^{d_1+d_2}_{+}$ and
\begin{equation}
\label{eq:matrix_v}
V(a, H, M) = a + r\lambda_1\prn*{H -{\textstyle\frac{1}{2}}\eta{}L^2 M} - \frac{c}{\eta},
\end{equation}
form a sufficient statistic pair for the adaptive regret bound $\cA_{\eta}$.
\end{proposition}
  
Now that we proposed a sufficient statistic, \pref{lem:suff_to_martingale} and \pref{lem:equivalence_burkholder} give a specific form of a martingale inequality and a construction for the special function (if the martingale inequality holds). Since the function constructed in the proof of \pref{lem:equivalence_burkholder} may not be efficiently computable, we embark on a search for a function that \emph{can} be evaluated efficiently. The next theorem presents such a Burkholder function. The proof rests on Lieb's Concavity Theorem \citep{lieb1973convex}, which states that for any fixed $A\in\sym^{d}$, the function $X \mapsto\Tr\,\exp\prn*{A + \log{}X}$ is concave over $\sym^{d}_{++}$.

  \begin{theorem}
    \label{thm:matrix_burkholder}
    Define $\burk:\reals\times \sym^{d_1+d_2}\times \sym^{d_1+d_2}_{+}\to{}\bbR$ via
    \[
      \burk(a, H, M) = a+ \frac{r}{\eta}\log\,\Tr\,\exp\prn*{\eta{}H - {\textstyle\frac{1}{2}}\eta^{2}L^{2}M} - \frac{c}{\eta}.
    \]
    Then $\burk$ is a Burkholder function, for the pair $(\suff, V)$ in \pref{eq:matrix_v} when $c\geq{}r\log(d_1 + d_2)$.

  \end{theorem}
  This Burkholder function construction immediately implies both existence of a prediction strategy (via \pref{lem:universal_algo}) and that a probabilistic inequality for matrix-values martingales holds. We will present both in detail. The matrix prediction strategy granted by the Burkholder algorithm is particularly simple due to extra convexity properties of $\burk$; see \pref{app:efficient}.
  
  \begin{corollary}[Matrix prediction algorithm]
    \label{corr:matrix_strategy}
    Suppose that $\cY=\brk*{-B, B}$ for some $B>0$. Then the deterministic strategy
    \begin{equation}
      \label{eq:matrix_pred}      
      \pred_{t} = \mathrm{proj}_{\brk*{-B, B}}\prn*{-\frac{r}{L\eta}\En_{\sigma\in\pmo}\brk*{
        \sigma\log\,\Tr\,\exp\,\prn*{\eta\sigma{}L\cH(X_t) + \eta{}\sum_{s=1}^{t-1}\dl_{s}\cH(X_s) - {\textstyle\frac{1}{2}}\eta^{2}L^{2}\sum_{s=1}^{t}\cM(X_s)}
        }}
      \end{equation}
       leads to a regret bound of
      \[
        \sum_{t=1}^{n}\loss(\pred_t, y_t) - \inf_{W\in\cW}\sum_{t=1}^{n}\loss(\tri*{W,X_t}, y_t) \leq{}
        {\textstyle\frac{1}{2}}\eta{}L^2 r\nrm*{\sum_{t=1}^{n}\cM(X_t)}_{\sigma} + \frac{r\log(d_1 + d_2)}{\eta}.
      \]
  \end{corollary}

  Since this regret bound is monotonically increasing with time, it is easy to tune $\eta$ to obtain a fully adaptive strategy.
  \begin{proposition}
    Let $R=\max_{t}\nrm*{X_t}_{\sigma}$ be known. By tuning $\eta$ through the standard doubling trick, we arrive at a regret bound of  
    \begin{align*}
      &\sum_{t=1}^{n}\loss(\pred_t, y_t) - \inf_{W\in\cW}\sum_{t=1}^{n}\loss(\tri*{W,X_t}, y_t) 
    \\ &\leq{}
      O\prn*{r\sqrt{\max\crl*{\nrm*{\sum_{t=1}^{n}X_{t}X_{t}^{\trn}}_{\sigma}, \nrm*{\sum_{t=1}^{n}X_{t}^{\trn}X_{t}}_{\sigma}}\log(d_1+d_2)} + Rr\log(n)}.
    \end{align*}
  \end{proposition}
	
	Let us briefly discuss the result. First, the computation in \pref{eq:matrix_pred} involves an SVD, and does not scale with $t$ since the method only keeps in memory the cumulative statistics. The regret bound gives a \emph{sequence-optimal} rate for the problem of \emph{Online Matrix Completion}, where each $X_t$ is an indicator $e_{i_t}e_{j_t}^{\top}$ corresponding to---for example---a user-movie pair for which the learner must predict a score. Here the regret bound obtained by \pref{eq:matrix_pred} interpolates between the worst-case configuration of the entries $(i_t,j_t)$ and ``spread-out'' (e.g. uniform) sampling of the entries. The result improves on \citep{foster2017zigzag}, which showed that this type of bound is possible by invoking the UMD inequality for Schatten norms but did not provide an efficient algorithm. See that paper for further discussion of the setting and problem.

  We now deliver on the second promise, namely a probabilistic martingale inequality. This inequality is stated for $\bbR^{d_1+d_2}$-valued Paley-Walsh martingale difference sequences $(\eps_t\mb{X}_t(\eps))_{t\leq{}n}$, where each $\mb{X}_{t}(\eps) = \mb{X}_{t}(\eps_1,\ldots,\eps_{t-1})$ is predictable with respect to $\cF_{t-1}=\sigma(\eps_{1},\ldots,\eps_{t-1})$ for Rademacher random variables $\eps_{1},\ldots,\eps_{n}$.

  \begin{corollary}[Martingale Matrix Square Function Inequality]
    \label{corr:matrix_square}
    For all Paley-Walsh martingale difference sequences $(\eps_t\mb{X}_{t}(\eps))_{t\leq{}n}$ it holds that
    \begin{equation}
      \label{eq:matrix_square}
    \En_{\eps}\nrm*{\sum_{t=1}^{n}\eps_t\mb{X}_{t}(\eps)}_{\sigma}
    \leq{} \sqrt{2\En_{\eps}\max\crl*{\nrm*{\sum_{t=1}^{n}\mb{X}_{t}(\eps)\mb{X}_{t}(\eps)^{\trn}}_{\sigma}, \nrm*{\sum_{t=1}^{n}\mb{X}_{t}(\eps)^{\trn}\mb{X}_{t}(\eps)}_{\sigma}}\log(d_1+d_2)}.
    \end{equation}
  \end{corollary}
    In the special case where $\mb{X}_t(\eps)=X_t$ is a fixed sequence, this square function inequality \pref{eq:matrix_square} recovers the Matrix Khintchine inequality \citep{mackey2014matrix}, including constants.  A similar martingale inequality can be obtained from the Matrix Freedman/Bennett inequalities of \cite{tropp2011freedman}, but this will depend on almost sure bounds on spectral norms of $(\mb{X}_{t}(\eps))_{t\leq{}n}$.

%%% Local Variables:
%%% mode: latex
%%% TeX-master: "paper"
%%% End:

\section{Further Examples}
\label{sec:further}
% !TEX root = paper.tex

\subsection{ZigZag Algorithm and the UMD Property}

\cite{Pisier75} used martingale techniques to provide a characterization of super-reflexive Banach spaces as those admitting an equivalent uniformly convex norm. As already described in Example~\ref{ex:smoothness}, the essential ingredient of this analysis is a construction of a function $\burk$ with the desired restricted concavity property (which turns out to be equivalent to uniform smoothness) for the martingale inequality \pref{eq:smoothness_martingale_ineq}. The corresponding notion in the world of online learning is that of an adaptive gradient (or mirror) descent. 

\cite{burkholder1981geometrical} provided a geometrical characterization of UMD spaces, and a key ingredient of the approach was to establish existence of (and sometimes to compute in closed form) the function $\burk$ with corresponding geometric properties ($\zeta$-convexity, which is equivalent to  ``zigzag concavity'' \citep{osekowski2012sharp}). As shown in \citep{foster2017zigzag}, in the online learning world the corresponding adaptive regret bound is that of empirical Rademacher averages:
$$\sum_{t=1}^n\loss(\pred_t,y_t) - \min_{\norm{w}\leq 1}\sum_{t=1}^n \loss(\inner{w,x_t},y_t) - C\En\norm{\sum_{t=1}^n \epsilon_t \delta_t x_t}.$$
By linearizing the loss, it suffices to use the sufficient statistic $\suff(x_t,\pred_t,\delta_t)= (\delta_t\pred_t, \delta_t x_t, \epsilon_t x_t)$ where $(\epsilon_t)$ is taken to be a sequence drawn by the algorithm.
The corresponding martingale inequality is
\begin{align}
	\label{eq:umd_martingale_ineq}
	\En\left[ \norm{\sum_{t=1}^n \epsilon_t\bx_t(\epsilon)}^p - C\norm{\sum_{t=1}^n \epsilon_t'\bx_t(\epsilon)}^p\right] \leq 0,
\end{align}
where the process in the subtracted term is decoupled and $p>1$ is arbitrary. We refer the reader to \citep{foster2017zigzag} for more details. 

We would like to emphasize that both smoothness/strong convexity (as in Pisier's work) and the UMD property (as in Burkholder's work) are two distinct notions with distinct sets of sufficient statistics. Since the fundamental works of Pisier and Burkholder, the so-called ``Burkholder method'' has been employed to prove a wide range of martingale inequalities and discover the corresponding geometric properties of the special function \citep{osekowski2012sharp,hytonen2016analysis}. The goal of this paper is to present a unifying approach for working with arbitrary sufficient statistics in online learning, and to show that the Burkholder approach is in fact \emph{algorithmic}.

\subsection{AdaGrad and Square Function Inequalities}
\label{sec:adagrad}

The Burkholder method can be used to recover efficient algorithms that obtain regret bounds in the vein of diagonal AdaGrad and full-matrix AdaGrad \citep{duchi2011adaptive}, with optimal constants. We thank Adam Os\k{e}kowski for suggesting this example to us \citep{osekowski2017personal}.

Define a function $\burk_{\textrm{square}}(x, y):\bbR^{d}\times{}\bbR_{+}\to\bbR$  \citep{osekowski2005two,osekowski2012sharp} via
\[
\burk_{\textrm{square}}(x, y) = \left\{
\begin{array}{ll}
-\sqrt{2y^{2} - \nrm*{x}_{2}^{2}},\quad& y\geq{}\nrm*{x}_{2}. \\
\nrm*{x}_{2}-2y,\quad& y<\nrm*{x}_{2}.
\end{array}
\right.
\]
$\burk_{\textrm{square}}$ satisfies three properties in the vein of \pref{lem:equivalence_burkholder}: 1. $\burk_{\textrm{square}}(x,y)\geq{}\nrm*{x}_{2}-2y$, 2. $\burk_{\textrm{square}}(x,\nrm*{x}_{2})\leq{}0$, and 3. $\burk_{\textrm{square}}(x+d,\sqrt{y^2 + \nrm*{d}_{2}^{2}})\leq{} \burk_{\textrm{square}}(x,y) + \tri*{\partial_{x}\burk_{\textrm{square}}(x,y), d}$. This function consequently leads to two algorithms in the style of AdaGrad \citep{duchi2011adaptive} but with optimal constants, and which we now sketch.

The first regret bound is for $\ls_{2}$ classes, as in full-matrix AdaGrad, and has the form
\[
\sum_{t=1}^n\loss(\pred_t,y_t) - \min_{\norm{w}_{2}\leq 1}\sum_{t=1}^n \loss(\inner{w,x_t},y_t) - 2L\sqrt{\sum_{t=1}^{n}\nrm*{x_t}^{2}_{2}}\leq{}0.
\]

The associated martingale inequality is 
$
\En\norm{\sum_{t=1}^n \epsilon_t\bx_t(\epsilon)}_{2} \leq{} 2\En\sqrt{\sum_{t=1}^{n}\nrm*{\bx_{t}(\eps)}^{2}_{2}},
$
which was shown to be optimal in \cite{osekowski2005two}.\footnote{Note that the expectation is outside the square root, so this is stronger than the ubiquitous inequality $\En\norm{\sum_{t=1}^n \epsilon_t\bx_t(\epsilon)}_{2} \leq{} \sqrt{\En\sum_{t=1}^{n}\nrm*{\bx_{t}(\eps)}^{2}_{2}}$.} The second regret bound is for $\ls_{\infty}$ classes, as in diagonal AdaGrad, and has the form
\[
\sum_{t=1}^n\loss(\pred_t,y_t) - \min_{\norm{w}_{\infty}\leq 1}\sum_{t=1}^n \loss(\inner{w,x_t},y_t) - 2L\nrm*{\prn*{\sum_{t=1}^{n}x_{t}^{2}}^{1/2}}_{1}\leq{}0,
\]
where $x_{t}^{2}$ is the element-wise square. This is obtained by applying the scalar version of $\burk_{\textrm{square}}$ coordinate-wise. The associated martingale inequality is 
$
\En\norm{\sum_{t=1}^n \epsilon_t\bx_t(\epsilon)}_{1} \leq{} 2\En\nrm*{\prn*{\sum_{t=1}^{n}\bx_{t}(\eps)^{2}}^{1/2}}_{1}.
$
Both regret bounds require no prior knowledge of the range of $(x_t)_{t\leq{}n}$.

\subsection{Strongly Convex Losses}
\label{sec:square_loss}
In this section we take $\cF=\crl*{x\mapsto{}\tri*{w,x}\mid{}w\in\bbR^{d}}$ and equip this space with a regularizer $\Phi(w) = \frac{1}{2}\nrm*{w}_{2}^{2}$. We assume that the loss $\ls(\yh, y)$ is $\rho$-strongly convex and $L$-Lipschitz. We adopt the shorthand $z_t=(x_t, -\yh_t)$, and our goal is to obtain a data- and comparator- dependent regret bound of the form

\[
\cA_{\lambda}(w; z_{1},\ldots,z_{n}) = \Phi((w,1)) + c\log\,\det\prn*{\rho{}\sum_{t=1}^{n}z_tz_t^{\trn} + \lambda{}I} - c\log\,\det(\lambda{}I).
\]
for some $c>0$. Here we the classical Vovk-Azoury-Warmuth-type bound for strongly convex losses \citep{Vovk98,AzouryWarmuth01}. This example is important because it shows that the Burkholder method in full generality can both obtain fast rates for curved losses and obtain bounds that jointly depend on the comparator and data; the UMD-type Burkholder functions used in \cite{foster2017zigzag} do not obtain such results. The right sufficient statistic for this problem should be familiar: In addition to storing a sum of gradients, we also store the empirical covariance $\sum_{t=1}^{n}z_tz_t^{\trn}$. We introduce one last piece of notation: For $A\succeq{}0$, $\Psi_{A}(w)=\frac{1}{2}\tri*{w,Aw}$.

\begin{proposition}
  \label{prop:square_loss_sufficient}
   The sufficient statistic $\suff(x_t,\pred_t,\delta_t)= \left( \delta_{t}z_t, z_tz_t^{\trn} \right)\in\bbR^{d+1}\times{}\sym^{d+1}_{+}$ and
\begin{equation}
\label{eq:vovk_azoury_warmuth_v}
V(x, A) = \Psi^{\star}_{\rho{}A + \lambda{}I}\prn*{x} - c\log\prn*{\det(\rho{}A + \lambda{}I)/\det(\lambda{}I)}
\end{equation}
forms a sufficient statistic pair for the adaptive regret bound $\cA_{\lambda}$.
\end{proposition}

\begin{theorem}
  \label{thm:square_loss_burkholder}
      For the sufficient statistic pair $(\suff, V)$ in \pref{prop:square_loss_sufficient}, $\burk=V$ is a Burkholder function whenever $c\geq{}L^{2}/\rho$. 
\end{theorem}
Note that for this setting the natural choice for $V$ turned out to be a Burkholder function itself.

\section{Necessary Conditions}
\label{sec:necessary}
% !TEX root = paper.tex

We now state a simple, yet powerful result that characterizes when existence of a Burkholder function for a sufficient statistic representation pair $(\suff,V)$ is not only sufficient, but \emph{necessary} to obtain a particular regret bound.
\begin{proposition}
\label{prop:lb}
Let $\delta=(\delta_1,\ldots,\delta_n)$ be a $[-L,L]$-valued martingale difference sequence over filtration $\F_{t-1}=\sigma(\delta_1,\ldots,\delta_{t-1})$ and let $\bz=(\bz_1,\ldots,\bz_n)$ be a sequence of functions $\bz_t: [-L,L]^{t-1} \to \X \times \Y$, each viewed as a predictable process with respect to $\F_{t-1}$. Suppose for every such $(\delta, \bz)$ pair there exists a randomized adversary strategy $(x_t, y_t)$ that guarantees, for every learner strategy $(\yh_t)_{t\leq{}n}$,
\begin{equation}
\label{eq:lb}
\En\sup_{f\in\cF}\brk*{\sum_{t=1}^n\loss(\pred_t,y_t)-\loss(f(x_t),y_t) - \cA(f; x_{1},\ldots,x_{n})}
\geq{} \En\left[  V\left(\sum_{t=1}^n \suff(\bz_t, \delta_t) \right) \right].
\end{equation}
Then, if there exists a strategy $(\yh_t)_{t\leq{}n}$ that achieves the regret bound $\cA(f;\xr[n])$, this implies that
\[
\sup_{\delta, \mb{z}}\En\left[  V\left(\sum_{t=1}^n \suff(\bz_t, \delta_t) \right) \right]\leq{}0.\footnote{In the more general case, if \pref{eq:lb} holds up to additive slack $\Delta$, the corresponding condition is $\sup\En\brk*{V}\leq{}\Delta$.}
\]
Consequently, the regret bound $\cA(f;\xr[n])$ is achievable only if there exists a Burkholder function $\burk:\cT\to\bbR$ that satisfies properties \propone/\proptwo/\propthree of \pref{lem:equivalence_burkholder}. 

When $\alpha \mapsto V(\tau+\suff(z,\alpha))$ is convex for any $z\in\X\times\Y,\tau\in\T$, we only require the preceeding inequalities to hold for $\delta_t=\epsilon_t \cdot L$, $\forall{}t=1,\ldots,n$, where $\epsilon_t$s are independent Rademacher random variables. In this case achievability of the regret bound $\cA(f;\xr[n])$ only implies existence of a Burkholder function $\burk$ satisfying property \propthreep{}, not \propthree{}.
\end{proposition}
\paragraph{Linear Classes}
At first glance the conditions of \pref{prop:lb} may seem fairly restrictive, but it is fairly straightforward to instantiate for all the examples in this paper. Consider the following linear setting: Take $\cX\subseteq{}\bbR^{d}$, $\cY$ arbitrary, and let $\cF$ be a linear class of the form $\crl*{x\mapsto{}\tri*{w,x}\mid{}w\in\cW}$, where $\sup_{x\in\cX,w\in\cW}\tri*{w,x}\leq{}1$ and $\cW$ is symmetric. Pick an arbitrary vector space $\overline{\cT}$, let $\overline{\suff}:\cX\to\overline{\cT}$ be an any featurization of the input space, and let $F:\overline{\cT}\to\bbR$ be an arbitrary function. Our goal will be to achieve a regret bound of the form 
\begin{equation}
\label{eq:suff_example}
\sum_{t=1}^{n}\ls(\yh_t, y_t) - \inf_{f\in\cF}\sum_{t=1}^n \loss(f(x_t),y_t)\leq{} \cA(\xr[n]) \ldef F\prn*{\sum_{t=1}^{n}\overline{\suff}(x_t)}.
\end{equation}
Let us first consider a natural choice of $V$ for the upper bound in this setting. Linearizing and using symmetry of $\cW$, we have
\[
\sum_{t=1}^{n}\ls(\yh_t, y_t) - \inf_{f\in\cF}\sum_{t=1}^n \loss(f(x_t),y_t) - \cA(\xr[n])
\leq{} \sum_{t=1}^{n}\yh_t\cdot\dl_t + \sup_{w\in\cW}\tri*{w,\sum_{t=1}^{n}\dl_{t}x_t} - F\prn*{\sum_{t=1}^{n}\overline{\suff}(x_t)}.
\]
This means that if we choose a sufficient statistic $\suff: (x_t, \yh_t, \dl_{t})\mapsto{} (\yh_t\dl_t, x_t\dl_t, \overline{\suff}(x_t))\in{}\bbR\times{}\bbR^{d}\times{}\overline{\cT}$ and choose $V(a, x, \overline{\tau})=a + \sup_{w\in\cW}\tri*{w,x} - F(\overline{\tau})$, then it holds that
\[
\sum_{t=1}^{n}\ls(\yh_t, y_t) - \inf_{f\in\cF}\sum_{t=1}^n \loss(f(x_t),y_t) - \cA(\xr[n]) \leq{} V\prn*{\sum_{t=1}^{n}\suff(x_t, \yh_t, \dl_t)}.
\]
Noting that $\alpha\mapsto{}V(\tau + \suff(x, \yh, \alpha))$ is convex, \pref{lem:suff_to_martingale} implies that a sufficient condition to achieve the regret bound for any convex $1$-Lipschitz loss is that
\begin{equation}
\label{eq:v_regret}
\sup_{\mb{z}}\En_{\eps}\brk*{V\prn*{\sum_{t=1}^{n}\suff(\mb{z}_t, \eps_t)}}\leq{}0,
\end{equation}
where $\mb{z}$ is any $\cX\times{}\cY$-valued predictable process with respect to the Rademacher sequence $\eps_{1},\ldots,\eps_{n}$. 

By specializing to the absolute loss $\ls(\yh, y)=\abs*{\yh-y}$ and choosing an adversary that plays $y_{t}$ to be Rademacher random variables and $x_{t}$ to be any predictable sequence, it can be shown that \pref{eq:v_regret} is also \emph{necessary}; this is proven formally in the appendix. As a corollary, we derive the following result.
\begin{proposition}
\label{prop:necessary}
There exists a Burkholder function $\burk$ for the pair $(\suff, V)$ \emph{if and only if} the regret bound \pref{eq:suff_example} is achievable.
\end{proposition}
 Consider the matrix prediction setting of \pref{sec:matrix} for the special case of $L=1$ and $r=1$. This setting fits into the linear class framework above by taking $\cW$ to be the nuclear norm ball in $\bbR^{d_1\times{}d_2}$ and setting $\overline{\suff}(X)=\cM(X)$ for any matrix $X\in\bbR^{d_1\times{}d_2}$. For this setting \pref{prop:necessary} implies the following equivalence.
\begin{example}[Matrix Prediction]
The following are equivalent:
\begin{enumerate}
\item The regret bound
\[
\sum_{t=1}^{n}\ls(\yh_t, y_t) - \inf_{W\;:\;\nrm*{W}_{\Sigma}\leq{}1}\sum_{t=1}^n \loss(\tri*{W,X_t},y_t) \leq{} \frac{\eta{}}{2}\nrm*{\sum_{t=1}^{n}\cM(X_t)}_{\sigma} + \frac{c}{\eta}
\]
is achievable.
\item The martingale inequality
\[
\En_{\eps}\nrm*{\sum_{t=1}^{n}\eps_t\mb{X}_{t}(\eps)}_{\sigma} \leq{} \frac{\eta{}}{2}\En_{\eps}\nrm*{\sum_{t=1}^{n}\cM(\mb{X}_t(\eps))}_{\sigma} + \frac{c}{\eta}
\]
holds for all $\bbR^{d_1\times{}d_2}$-valued predictable processes $\mb{X}$.
\item There exists a Burkholder function for the sufficient statistic pair $(\suff, V)$ in \pref{eq:v_regret}.
\end{enumerate}

\end{example}

\section{Discussion}
\label{sec:discussion}
% !TEX root = paper.tex

The core techniques developed in this paper suggest a number of promising future directions and natural extensions.

\textbf{Finding sufficient statistics}~~~~ This paper gives multiple examples of Burkholder function constructions and sufficient statistics. If one wishes to find sufficient statistics for an adaptive bound $\cA$ of interest, a basic rule of thumb is to consider a single input instance (instead of all $n$ data points) and determine---say---a polynomial expansion or expansion in another basis for the terms in $\mathrm{Reg}_{n}-\cA$ involving the instance. This gives a coarse sketch of which statistics are necessary. 

As an example, take the standard square loss with linear predictors as the benchmark class and suppose we are interested in a non-adaptive bound. Following the heuristic above, we need to find an expansion for terms of the form $(\hat{y} - y)^2 - (\tri*{w,x} - y)^2 - ~\mathrm{constant}$. Expanding this expression out, we find that $\hat{y}^2$, $y \cdot x$ and $x x^\top$ are all required to write the expression explicitly. In fact, for this square loss example, the weighted sum of the $x_t$s and the sum of the outer products $\sum_t x_t x_t^\top$ turn out to be sufficient statistics as well.  

For the examples in this paper, we exclusively considered benchmark classes $\F$ that were linear, which appears to have made the search for sufficient statistics easier. However, even when one considers a class $\F$ of non-linear functions, the approach of trying to expand the desired regret inequality (which now involves nonlinear $f \in \F$) around a given instance $x$ in terms of some basis may still help to obtain an adequate sufficient statistics. Furthermore, one may enlarge the class $\F$ to make the sufficient statistic search easier. For instance, if we want to learn the class of boolean decision trees of depth $d$, we can exploit that the class can be represented by polynomials of degree $d$ by using the discrete Fourier coefficients of the input instances up to degree $d$ as a sufficient statistic. In summary, for non-linear classes one may still search for sufficient statistics and Burkholder functions by expressing nonlinearities (approximately) via linear combinations of higher-order terms. 

\textbf{Toward plug-and-play online learning}~~~~
A natural next step is to automatize the search for sufficient statistics and \Bfun functions. Suppose that the sufficient statistic pair $(\suff, V)$ is fixed and all that remains is to find a Burkholder function $\burk$. If $V$ can be written as a polynomial of degree over sufficient statistic space $\cT$, a natural approach is to restrict the search to Burkholder functions $\burk$ that are themselves polynomials and relax the inequalities \propone/\proptwo/\propthree{} to sum-of-squares constraints \citep{barak2014sum}. We can then jointly search for a function $\burk$ and a degree-$d$ sum-of-squares proof that this function satisfies the three properties in polynomial time once the degree of $\burk$ is fixed. As a specific example, the problem of finding the zig-zag concave Burkholder function for $\ell_p$ norms explored in \cite{foster2017zigzag} has a sufficient statistic $V$ that is a polynomial of degree $p$ when $p\geq{}2$ is an integer. 

This approach is sound in that it will never incorrectly return a function $\burk$ that does not satisfy the three properties, but may not be complete a-priori. An interesting direction is therefore to explore whether there are conditions under which this system can indeed be made complete.

\textbf{Generalized/non-additive sufficient statistics}~~~~ The restriction in \pref{def:sufficiency} that sufficient statistics combine additively can be relaxed. A more general form is as follows. First, define a \emph{representation space} $\cT$. The function $\suff$ now takes the form:
\[
  \suff: \cX\times{}\cY\times{}\brk*{-L, L}\times{}\cT \to \cT.
\]
The restricted concavity condition for $\burk$ under this definition becomes
\[
\forall{}z, \tau:\quad\sup_{\En\brk*{\alpha}=0}\En_{\alpha}\burk\prn*{\suff\prn*{z, \alpha, \tau}} \leq{} \burk(\tau).
\]
Properties \propone{} and \proptwo{} of \pref{lem:equivalence_burkholder} remain the same. This generalized notion of a sufficient statistic allows us to move beyond additive updates --- $\suff$ can multiply $z$ with elements of $\cT$, for example --- but still restricts storage to the space $\cT$ and is fully compatible with the Burkholder method and general algorithm framework. The generalizations of the equivalence theorem (\pref{lem:equivalence_burkholder}) and the Burkholder algorithm (\pref{lem:universal_algo}) for this notion of sufficient statistic hold as well.

\section*{Acknowledgements}
We thank Adam Os\k{e}kowski  for helpful discussions and for suggesting the example in \pref{sec:adagrad}.

\bibliography{refs}

\appendix

\section{Proofs}
\label{app:proofs}
% !TEX root = paper.tex

\subsection{Proofs from \pref{sec:burkholder} and \pref{sec:algorithm}}

\begin{proof}[\pfref{lem:suff_to_martingale}]
We will use the notation $\multiminimax{\ldots}_{t=1}^n$ to denote the repeated application of operators, with the outer application corresponding to $t=1$. Existence of a randomized strategy for \pref{eq:def_phi_regret} is equivalent to the following quantity being non-positive:
\begin{align*}
&\multiminimax{\sup_{x_t\in\cX} \inf_{q_t\in\Delta_{\cY}} \sup_{y_t\in\cY} \En_{\yh_t\sim q_t}}_{t=1}^n \left[\sum_{t=1}^n \loss(\pred_t,y_t) - \phi(x_1,y_1,\ldots,x_n,y_n)  \right].
\end{align*}
By the minimax theorem, this is equal to 
\begin{align*}
&\multiminimax{ \sup_{x_t\in\cX} \sup_{p_t\in\Delta_{\cY}} \inf_{\pred_t\in\cY}\En_{y_t \sim p_t}}_{t=1}^n \brk*{ \sum_{t=1}^n \loss(\pred_t,y_t) - \phi(x_1,y_1,\ldots,x_n,y_n)}.
\end{align*}
See \citep{RakSriTew10,rakhlin2012relax,FosRakSri15} for detailed discussion of the technical conditions under which the minimax theorem can be applied in the online learning setting; briefly, our assumptions that $\cY$ is a compact subset of $\bbR$ and that $\ls$ and $\phi$ are bounded are sufficient.
In view of \pref{eq:suffiency_def}, the above quantity is upper bounded by 
\begin{align*}
&\leq{} \multiminimax{\sup_{x_t\in\cX} \sup_{p_t\in\Delta_{\cY}} \inf_{\pred_t\in\cY}\En_{y_t \sim p_t}}_{t=1}^n \brk*{ V\left(\sum_{t=1}^n \suff(x_t,\pred_t,\partial \loss(\pred_t,y_t))\right) }.
\end{align*}
 Now, for each time $t$, choose the dual strategy $\pred_t^* \ldef \arg\min_{\pred\in\cY} \En_{y_t \sim p_t}\loss(\pred,y_t)$ so that $0\in\partial \En_{y_t\sim p_t} \loss(\pred^*_t, y_t)$; that this is possible is implied by the assumption on the loss $\ls$ stated in \pref{sec:problem}.
This choice implies that $\partial \loss(\pred_t^*,y_t) = \delta_t$ is a zero mean real variable conditionally on the past, i.e. $\En\brk*{\delta_{t}\mid{}\cG_{t}}=0$, where $\cG_{t}=\sigma(\yh_{1:t-1})$.
This particular choice for the $\yh_t$ in the dual game leads to the upper bound 
\begin{align*}
&\multiminimax{ \sup_{x_t\in\cX} \sup_{p_t\in\Delta_{\cY}} \En_{y_t \sim p_t}}_{t=1}^n \brk*{ V\left(\sum_{t=1}^n \suff(x_t,\pred_t^*,\delta_t)\right) },
\end{align*}
which is, in turn, upper bounded by 
\begin{align*}
&\multiminimax{ \sup_{z_t\in \X\times \cY} \sup_{p_t\in\Delta_{\brk*{-L, L}}: \En[\delta_t]=0} \En_{\delta_t \sim p_t}}_{t=1}^n \brk*{ V\left(\sum_{t=1}^n \suff(z_t,\delta_t)\right) }.
\end{align*}
The last expression can be written in the functional form as
$$
\sup_{\bz, \bp} \mathbb{E}_{\delta \sim \bp}\left[  V\left(\sum_{t=1}^n \suff(\bz_t, \delta_t)\right) \right].
$$
using the notation of the lemma, with the supremum over $\bp$ ranging over all joint distributions on $\delta=(\delta_1,\ldots,\delta_n)$ satisfying $\En[\delta_t\mid{}\delta_{1:t-1}]=0$ for all $t\in[n]$. The non-positivity of the latter quantity is therefore sufficient to ensure the existence of a prediction strategy satisfying \pref{eq:def_phi_regret}.
\end{proof}

\begin{proof}[\pfref{lem:equivalence_burkholder}]
We first establish existence of $\burk$ under the premise of the lemma. The construction is given by
\begin{equation}
\label{eq:u_construction}
\burk(\tau) = \sup_{\bz,  \bp} \En_{\delta \sim \bp}\left[ V\left(\tau + \sum_{t\geq 1} \suff(\bz_t,\delta_t) \right)\right].
\end{equation}
Then under the probabilistic inequality that is the premise of the lemma, it holds that
\[
\burk(0) = \sup_{\bz,  \bp} \En_{\delta \sim \bp}\left[ V\left(\sum_{t\geq 1} \suff(\bz_t,\delta_t) \right)\right] \le 0.
\]
 Next, by our assumption, $\exists z^0 $ s.t. $\suff(z^0,0) = 0$, we can lower bound the supremum in \pref{eq:u_construction} by considering a particular $\bz$ that is constant $\bz_t\ldef z^0$ for all $t$, and a distribution for $\delta_t$ that only places mass on the singleton $0$. This yields a lower bound
$$
\burk(\tau) \ge V(\tau).
$$
To verify the third condition, observe that for any zero-mean random variable $\alpha$ with distribution $p$ supported on $[-L,L]$,
\begin{align*}
\En_{\alpha}\left[\burk(\tau + \suff(z,\alpha))\right] &=  \En_{\alpha}\left[ \sup_{\bz, \bp} \En_{\delta \sim \bp}\left[ V\left(\tau + \suff(z,\alpha) + \sum_t \suff(\mathbf{z}_t, \delta_t) \right)\right] \right]\\
 &\le   \sup_{\bz, \bp} \En_{\delta \sim \bp}\left[ V\left(\tau  + \sum_t \suff(\bz_t, \delta_t) \right)\right] \\
 & = \burk(\tau).
\end{align*}
For the converse, assume we have a function $\burk$ satisfying the three properties. Fix any $\bz$ and $\bp$ of length $n$. In this case, by property \proptwo, the following inequality holds deterministically:
\begin{align*}
 V\left(\sum_{t=1}^n \suff(\bz_t,\delta_t)\right)  & \le \burk\left(\sum_{t=1}^n \suff(\bz_t,\delta_t)\right).
\end{align*}
By property \propthree, we have that for any time $s$,
\begin{align*}
\En_{\delta_n} \burk\left(\sum_{t=1}^s \suff(\bz_t,\delta_t)\right) \le   \burk\left(\sum_{t=1}^{s-1} \suff(\bz_t,\delta_t)\right).
\end{align*}
Continuing this argument all the way to $t=0$ and using property \propone,
\begin{align*}
\sup_{\bz, \bp} \En_{\delta \sim \bp}\left[  V\left(\sum_{t=1}^n \suff(\bz_t,\delta_t)\right) \right] \le \burk(0) \le 0.
\end{align*}
\end{proof}

% !TEX root = paper.tex

\subsection{Proofs from \pref{sec:linear_loss}}

\begin{proof}[\pfref{prop:param_sufficient}]
  We define a potential function that will eventually be used in the construction of the Burkholder function $\burk$ we provide for $V$. As discussed in the main body, a variant of this potential was first introduced by \cite{mcmahan2014unconstrained} for the special case of Hilbert spaces. Let $\Psi(x) = \frac{1}{2}\nrm*{x}^{2}$ (not necessarily a Hilbert space norm) and define
  \[
    F_{n}(x) = \gamma\exp\prn*{\frac{\Psi(x)}{an}}.
  \]
  From \cite[Lemma 14]{mcmahan2014unconstrained}, along with the additional fact that $(f(\nrm*{\cdot}))^{\star} = f^{\star}(\nrm*{\cdot}_{\star})$ for general dual norm pairs, it holds that
  \[
    F_{n}^{\star}(w) \leq{} \nrm*{w}_{\star}\sqrt{2an\log\prn*{\frac{\sqrt{an}\nrm*{w}_{\star}}{\gamma} + 1}}.
  \]
  This is all we need to establish the result. We proceed as follows
\begin{align*}
&\sup_{w\in\bbR^{d}}\crl*{
    \reg(w) - \cA(w)
      } \\
      &=
      \sup_{w\in\bbR^{d}}\crl*{
    \sum_{t=1}^{n}\ls(\pred_t, y_t) - \ls(\tri*{w,x_t}, y_t) - \cA(w)
      } \\
     &\leq{}
      \sup_{w\in\bbR^{d}}\crl*{
    \sum_{t=1}^{n}\partial\ls(\pred_t, y_t)(\pred_t - \tri*{w,x_t}) - \cA(x_{1},\ldots,x_{n})
      } \\
     &=
\sum_{t=1}^{n}\partial\ls(\pred_t, y_t)\cdot\pred_t  +  \sup_{w\in\bbR^{d}}\crl*{
       \tri*{w,\sum_{t=1}^{n}\partial\ls(\pred_t, y_t)x_t} - \cA(w)
       }
       \intertext{Using the inequality for the potential $F^{\star}_n$ stated above:}
      &\leq{}
        \sum_{t=1}^{n}\partial\ls(\pred_t, y_t)\cdot\pred_t  +  
        F^{\star}_n\prn*{\sum_{t=1}^{n}\partial\ls(\pred_t, y_t)x_t}
        - c
\end{align*}
It follows that $\suff(x_t,\pred_t,\delta_t) = \left( \delta_t\cdot\pred_t, \delta_t\cdot x_t \right)\in\reals\times \X$ is a sufficient statistic. This is because we can write
\[
V(b, x) = b + F^{\star}_n(x) - c.
\]
and we have just shown that
\[
\sup_{w}\crl*{\reg(w) - \cA(x_{1},\ldots,x_{n})
      }
      \leq{} V\prn*{
      \sum_{t=1}^{n}\suff(x_t,\pred_t,\delta_t)
      }.
\]

\end{proof}

\begin{proof}[\pfref{thm:param_free}]
    Since $\burk$ depends on time, we generalize the properties of \pref{lem:equivalence_burkholder} to
  \begin{enumerate}
  \item[$1^o$] $\burk_{0}(0) \le 0$
  \item[$2^o$] For any $\tau \in \T$, $\burk_{n}(\tau) \ge V(\tau)$
  \item[$3^o$] For any $\tau \in \T$, $z \in \X \times \Y$, and any mean-zero distribution $p$ on $[-L,L]$, and any $t\geq{}1$
    \begin{align}
      \En_{\alpha\sim p}\left[\burk_{t}(\tau + \suff(z,\alpha))\right] \le \burk_{t-1}(\tau) 
    \end{align}    
  \item[$3'$] For any $\tau \in \T$, $z \in \X \times \Y$, and any $t\geq{}1$,
    $$
    \forall \tau \in \T, z \in \X \times \Y,~~~ \En_\epsilon \burk_{t}(\tau + \suff(z,\epsilon L))  \le \burk_{t-1}(\tau),
    $$
    where $\epsilon$ is a Rademacher random variable. 
  \end{enumerate}

  Recall that for simplicity we assume $L=1$ and $\X$ is a unit ball: $\nrm*{x}\leq{}1$. Let $\Psi(x) = \frac{1}{2}\nrm*{x}^{2}$, where we have assumed that $\beta$-smoothness of $\Psi$:
  \[
    \Psi(x+y) \leq{} \Psi(x) + \tri*{\grad{}\Psi(x), y} + \frac{\beta}{2}\nrm*{y}^{2}.
  \]
  
  Define a family of potentials
  \[
    F_{t}(x) = \gamma\exp\prn*{\frac{\Psi(x)}{at} + \frac{1}{2}\sum_{s=t+1}^{n}\frac{1}{s}}
  \]
  and $F_{0} = \gamma\exp\prn*{\frac{1}{2}\sum_{t=1}^{n}\frac{1}{t}}$. Note that $F_{n}$ here is the same as in the proof of \pref{prop:param_sufficient}.
  
  Observe that
  \[
    \burk_{t}(b, x) = b + F^{\star}_t(x) - c, 
  \]
  where $F_{t}^{\star}$ is as defined as in the proof of \pref{prop:param_sufficient}. We proceed to establish the three properties of $\burk$ from \pref{lem:equivalence_burkholder}. Property $2^o$ holds since $V=\burk_{n}$. We will show property \propthreep{} first, then conclude with property \propone{}. Note that $\alpha\mapsto{}\burk_{t}(\tau + \suff(z,\alpha))$ is convex with respect to $\alpha$, and so it indeed suffices to show property \propthreep{}.

  Fix an element $\tau=(\tau_1, \tau_2)\in\bbR\times{}\cX=\cT$ of the sufficient statistic space. At time $n$ we have
  \[
    \En_{\eps}\left[\burk_{n}(\tau + \suff(z,\eps))\right] = \En_{\eps}\brk*{\tau_{1} +  \eps\cdot\pred + F_{n}(\tau_{2} + \eps{}x_n)} - c = \tau_{1} + \En_{\eps}\brk*{F_{n}(\tau_{2} + \eps{}x_n)} - c.
  \]
  To handle $F_n$, begin by using smoothness of $\Psi$:
  \begin{align*}
\En_{\eps}\brk*{F_{n}(\tau_{2} + \eps{}x_n)} =  \En_{\eps}\exp\prn*{\frac{\Psi(\tau_{2} + \eps{}x)}{an}} &\leq{}
                                                               \En_{\eps}\exp\prn*{\frac{\Psi(\tau_{2}) + \eps\tri*{\grad\Psi(\tau_2), x} + \frac{\beta}{2}\nrm*{x}^{2}}{an}} 
\end{align*}
Using the standard Rademacher mgf bound, $\En_{\eps}e^{\lambda{}\eps}\leq{}e^{\lambda^{2}/2}$, we upper bound the above quantity by
\begin{align*}
\exp\prn*{\frac{\Psi(\tau_{2}) + \frac{\beta}{2}\nrm*{x}^{2}}{an} + \frac{\tri*{\grad\Psi(\tau_2), x}^{2}}{2(an)^{2}}}\leq{}    \exp\prn*{\frac{\Psi(\tau_{2}) + \frac{\beta}{2}\nrm*{x}^{2}}{an} + \frac{\nrm*{\grad\Psi(\tau_2)}_{\star}^{2}\nrm*{x}^{2}}{2(an)^{2}}}.
\end{align*}
Using the assumption $\nrm*{x}\leq{}1$, we obtain an upper bound of
\begin{align*}
&\exp\prn*{\frac{\Psi(\tau_{2}) + \frac{\beta}{2}}{an} + \frac{\nrm*{\grad\Psi(\tau_2)}_{\star}^{2}}{2(an)^{2}}}.
\end{align*}
We now use a basic fact from convex analysis, namely that any $\beta$-smooth function $f$, $\frac{1}{2\beta}\nrm*{\grad{}f(x) - \grad{}f(y)}_{\star}^{2} \leq{} f(x) - f(y) - \tri*{\grad{}f(y), x-y}$ . This yields an upper bound
\begin{align*}
&\exp\prn*{\frac{\Psi(\tau_{2}) + \frac{\beta}{2}}{an} + \frac{\beta\Psi(\tau_2)}{(an)^{2}}}
\end{align*}
Setting $a=\beta$, this is equal to
\begin{align*} 
	\exp\prn*{\frac{1}{\beta}\prn*{\frac{1}{n} + \frac{1}{n^2}}\Psi(\tau_2) + \frac{1}{2n}}.
\end{align*}
  As a last step, observe that $\frac{1}{n} + \frac{1}{n^{2}} \leq{} \frac{1}{n-1}$. Indeed,
  \[
    \frac{1}{n} + \frac{1}{n^{2}} = \frac{1}{n}\prn*{1 + \frac{1}{n}} = \frac{1}{n-1}\frac{n-1}{n}\prn*{1+\frac{1}{n}}
    = \frac{1}{n-1}\prn*{1-\frac{1}{n}}\prn*{1+\frac{1}{n}}
    = \frac{1}{n-1}\prn*{1 - \frac{1}{n^{2}}}\leq{} \frac{1}{n-1}.
  \]
  Therefore, we have established that
  \[
    \En_{\eps}\brk*{F_{n}(\tau_{2} + \eps{}x_n)} \leq{} \exp\prn*{ \frac{\Psi(\tau_2)}{\beta(n-1)} + \frac{1}{2n}} = F_{n-1}(\tau_2),
  \]
  and in particular $\En_{\eps}\burk_{n}(\tau + \suff(z,\eps))\leq{}\burk_{n-1}(\tau)$.
  In fact, by folding the terms $\frac{1}{2}\sum_{s=t+1}^{n}\frac{1}{s}$ --- which do not depend on data --- into a multiplicative constant, this argument yields, for any $t$ and any $\nrm*{x}\leq{}1$,
  \[
    \En_{\eps}\brk*{F_{t}(\tau + \eps{}x)} \leq{} F_{t-1}(\tau).
  \]
  Thus, for each $t\geq{}2$ we have
  \[
    \En_{\eps}\left[\burk_{t}(\tau + \suff(z,\eps))\right] = \En_{\eps}\brk*{\tau_{1} +  \eps\cdot\pred + F_{n}(\tau_{2} + \eps{}x)} - c \leq{} \burk_{t-1}(\tau).
  \]
  
The argument also yields (by removing unnecessary steps):
  \[
    \En_{\eps}\brk*{F_{1}(0 + \eps{}x)} \leq{} \gamma\exp\prn*{\frac{1}{2}\sum_{t=1}^{n}\frac{1}{t}} = F_{0}.
  \]
  This means that
  \[
    \burk_{0}(0) = \gamma\exp\prn*{\frac{1}{2}\sum_{t=1}^{n}\frac{1}{t}} - c \leq{} \gamma\exp\prn*{\log(n)/2} - c.
  \]
  We will set $\gamma=\frac{1}{\sqrt{n}}$ and $c=1$, which yields $\burk_{0}(0) \leq{} 0$.
  
\end{proof}

\subsection{Proofs from \pref{sec:matrix}}

\begin{proof}[\pfref{prop:matrix_sufficient}]

  Recall that $\cA_{\eta}(X_{1},\ldots,X_{n}) = \frac{\eta{}rL^{2}}{2}\nrm*{\sum_{t=1}^{n}\cM(X_t)}_{\sigma} + \frac{c}{\eta}$. Linearizing the loss with the adaptive bound as in \pref{eq:phi_comp_adap}, 
\begin{align*}
	&\sum_{t=1}^{n}\ls(\pred_t, y_t) - \inf_{W\in\cW}\ls(\tri*{W,X_t}, y_t) - \cA_{\eta}(X_{1},\ldots,X_{n}) \\
    &\leq \sup_{W\in\cW}\crl*{
       \sum_{t=1}^{n}\partial\ls(\pred_t, y_t)(\pred_t - \tri*{W,X_t}) - \cA_{\eta}(X_{1},\ldots,X_{n})
      } \\
    &=
      \sum_{t=1}^{n}\partial\ls(\pred_t, y_t)\pred_t  + r\nrm*{\sum_{t=1}^{n}\partial\ls(\pred_t, y_t)X_t}_{\sigma}- \cA_{\eta}(X_{1},\ldots,X_{n}) .
  \end{align*}
	We now abbreviate $\partial\ls(\pred_t, y_t) = \dl_{t}$ and expand out $\cA_{\eta}$, yielding
\begin{align*}
        &\sum_{t=1}^{n}\dl_t\cdot\pred_t  + r\nrm*{\sum_{t=1}^{n}\dl_tX_t}_{\sigma} - \frac{\eta{}rL^{2}}{2}\nrm*{\sum_{t=1}^{n}\cM(X_t)}_{\sigma} - \frac{c}{\eta}.
\end{align*}
Using the fact that $\lambda_{1}(\cH(X)) = \nrm*{X}_{\sigma}$, linearity of $\cH$, and that $\cM(X_t)$ is positive semidefinite, we write this as
\begin{align*}
      &\sum_{t=1}^{n}\dl_t\cdot\pred_t  + r\lambda_{1}\prn*{\sum_{t=1}^{n}\dl_t\cH(X_t)} - r\lambda_{1}\prn*{\frac{\eta{}L^2}{2}\sum_{t=1}^{n}\cM(X_t)} - \frac{c}{\eta}
\end{align*}
Sub-additivity of $\lambda_{1}$ gives a further upper bound of
\begin{align*}    
        \sum_{t=1}^{n}\dl_t\cdot\pred_t  + r\lambda_{1}\prn*{\sum_{t=1}^{n}\dl_t\cH(X_t)-\frac{\eta{}L^{2}}{2}\sum_{t=1}^{n}\cM(X_t)} - \frac{c}{\eta}
\end{align*}
Then $\suff(X_t,\pred_t,\delta_t) = \left( \delta_t\cdot\pred_t, \delta_t\cdot \cH(X_t), \cM(X_t) \right)\in\reals\times \sym^{d_1+d_2}\times \sym^{d_1+d_2}_{+}$ is a sufficient statistic. Namely, writing
\[
V(a, H, M) = a + r\lambda_1\prn*{H -\frac{\eta{L^{2}}}{2}M} - \frac{c}{\eta},
\]
our calculation shows that
\[
\sup_{W\in\cW}\crl*{\reg(W) - \cA(X_{1},\ldots,X_{n})
      }
      \leq{} V\prn*{
      \sum_{t=1}^{n}\suff(X_t,\pred_t,\delta_t)
      }.
\]
  \end{proof}

  \begin{proof}[\pfref{thm:matrix_burkholder}]

    Recall that
    \[
      \burk(a, H, M) = a+ \frac{r}{\eta}\log\,\Tr\,\exp\prn*{\eta{}H - \frac{\eta^{2}L^{2}}{2}M} - \frac{c}{\eta}
    \]
We will show that $\burk$ satisfies the three properties of \pref{lem:equivalence_burkholder}. For property \propone{}, we have
\[
\burk(0) = \frac{r}{\eta}\log\prn*{\Tr\prn*{\exp\prn*{0}}} - \frac{c}{\eta} = \frac{r\log(d_1 + d_2)}{\eta} - \frac{c}{\eta}.
\]
Thus, $\burk(0)\leq{}0$ as soon as $c\geq{}r\log(d_1+d_2)$.

For property \proptwo{}, it suffices to show that $\lambda_{1}(H-\frac{\eta{}L^{2}}{2}M) \leq{} \frac{1}{\eta}\log\,\Tr\,\exp\prn*{\eta{}H - \frac{\eta^{2}L^{2}}{2}M}$. To this end, we have
\begin{align*}
  \lambda_{1}\left(H-\frac{\eta{}L^{2}}{2}M \right)  =   \frac{1}{\eta}\log \lambda_{1}\prn*{\exp\prn*{\eta{}H-\frac{\eta^2L^{2}}{2}M}} \leq{} \frac{1}{\eta}\log\,\Tr\,\exp\prn*{\eta{}H-\frac{\eta^2L^{2}}{2}M
  },
\end{align*}
where the equality is well-defined because the matrix under consideration is symmetric and the inequality follows because $e^{A}$ is positive semidefinite for any symmetric matrix $A$.

For the third property, observe that the mapping $\alpha\mapsto{}V(\tau + \suff(z,\alpha))$ is convex (e.g. \citep{lewis1996convex}). Consequently, by \pref{lem:equivalence_burkholder}, it suffices only to prove property $3'$, i.e. that the restricted concavity condition holds only for Rademacher random variables.

Fix $\tau\in\cT$ and $z=(X, \pred)\in\cX\times{}\cY$, and let $\eps$ be a Rademacher random variable. Writing
\[
\tau = (\tau_{1}, \tau_{2}, \tau_{3})\in\reals\times \sym^{d_1+d_2}\times \sym^{d_1+d_2}_{+},
\]
we have
\begin{align*}
 \En_{\eps}\left[\burk(\tau + \suff(z,\eps{}L))\right] 
 &= \En_{\eps}\brk*{
   \tau_{1} + \pred\eps{}L + \frac{r}{\eta}\log\,\Tr\,\exp\prn*{\eta{}\tau_{2} - \frac{\eta^{2}L^{2}}{2}\tau_{3} + \eta\eps{}L\cH(X) - \frac{\eta^{2}L^{2}}{2}\cM(X)}
   } - \frac{c}{\eta}. \\
   &= \frac{r}{\eta}\En_{\eps}\brk*{
     \log\,\Tr\,\exp\prn*{\eta{}\tau_{2} - \frac{\eta^{2}L^{2}}{2}\tau_{3} + \eta\eps{}L\cH(X) - \frac{\eta^{2}L^{2}}{2}\cM(X)}
   } +      \tau_{1} - \frac{c}{\eta}.
\end{align*}
Focusing on the log-trace-exponential term, observe that
\begin{align*}
  &  \En_{\eps}\brk*{
  \log\,\Tr\,\exp\prn*{\eta{}\tau_{2} - \frac{\eta^{2}L^{2}}{2}\tau_{3} + \eta\eps{}L\cH(X) - \frac{\eta^{2}L^{2}}{2}\cM(X)}
} \\
&=  \En_{\eps}\brk*{
       \log\,\Tr\,\exp\prn*{\eta{}\tau_{2} - \frac{\eta^{2}L^{2}}{2}\tau_{3} + \log\prn*{\exp\prn*{\eta\eps{}L\cH(X)}} - \frac{\eta^{2}L^{2}}{2}\cM(X)}
       }.
\end{align*}
Since $\exp\prn*{\eta\eps{}L\cH(X)}$ is positive definite and $\eta{}\tau_{2} - \frac{\eta^{2}L^{2}}{2}\tau_{3} -\frac{\eta^{2}L^{2}}{2}\cM(X)$ is symmetric (by assumption), we can apply Lieb's Concavity Theorem to upper bound this by
\begin{align*}	   
  & \log\,\Tr\,\exp\prn*{\eta{}\tau_{2} - \frac{\eta^{2}L^2}{2}\tau_{3} + \log\prn*{\En_{\eps}\exp\prn*{\eta\eps{}L\cH(X)}} - \frac{\eta^{2}L^2}{2}\cM(X)}.
\end{align*}
The Rademacher matrix mgf bound \citep{tropp2012user} now yields
\[
\log\prn*{\En_{\eps}\exp\prn*{\eta\eps{}L\cH(X)}} \preceq \log\prn*{\exp\prn*{\eta^{2}L^{2}\cM(X)}/2} = \eta^{2}L^{2}\cM(X)/2.
\]
Since $A\preceq{}B$ implies $\Tr{}e^{A}\leq{}\Tr{}e^{B}$, this implies that
\[
\En_{\eps}\brk*{
  \log\,\Tr\,\exp\prn*{\eta{}\tau_{2} - \frac{\eta^{2}L^{2}}{2}\tau_{3} + \eta\eps{}L\cH(X) - \frac{\eta^{2}L^{2}}{2}\cM(X)}
}
\leq{}
  \log\,\Tr\,\exp\prn*{\eta{}\tau_{2} - \frac{\eta^{2}L^{2}}{2}\tau_{3}}
\]
Combining everything we proved so far, this implies
\[
  \En_{\eps}\left[U(\tau + \suff(z,\eps{}L))\right]
  \leq{} \tau_{1} + \frac{r}{\eta}  \log\,\Tr\,\exp\prn*{\eta{}\tau_{2} - \frac{\eta^{2}L^2}{2}\tau_{3}}  - \frac{c}{\eta} = U(\tau).
\]

\end{proof}

\begin{proof}[\pfref{corr:matrix_strategy}]
  The Burkholder function $\burk$ satisfies the conditions of \pref{lem:det_strat3}. Direct calculation shows that the strategy in \pref{lem:det_strat3} matches the strategy in the statement of the corollary.
\end{proof}

\begin{proof}[\pfref{corr:matrix_square}]
  We invoke the Burkholder function $\burk$ from \pref{thm:matrix_burkholder} for the special case $r=1$ and $c=\log(d_1 + d_2)$, and $L=1$. In particular, its existence per \pref{lem:equivalence_burkholder} implies (for the corresponding $V$, here denoted $V_{\eta}$ to refer to the $V$ given for a fixed value of $\eta$)
  \[
    \inf_{\eta}\sup_{\mb{z},\mb{p}, n}\En\brk*{V_{\eta}\prn*{\sum_{t=1}^{n}\suff(\mb{z}_t, \dl_{t})}}\leq{}0
  \]
  We use this inequality only for the special case where $\dl_{t}=\eps_{t}$ and $\mb{z}_t=(\mb{X}_{t}(\eps), 0)$. For this special case, the inequality above implies
    \[
      \inf_{\eta}\sup_{\mb{X}, n}\En\brk*{\nrm*{\sum_{t=1}^{n}\eps_t\mb{X}_{t}(\eps)} - \frac{\eta}{2}\nrm*{\sum_{t=1}^{n}\cM(\mb{X}_{t}(\eps))} - \frac{\log(d_1 + d_2)}
        {\eta} }\leq{}0
    \]
    For any fixed martingale $(\mb{X}_{t}(\eps))_{t\leq{}n}$, this implies
    \[
      \En\nrm*{\sum_{t=1}^{n}\eps_t\mb{X}_{t}(\eps)} \leq{} \inf_{\eta>0}\crl*{\frac{\eta}{2}\En\nrm*{\sum_{t=1}^{n}\cM(\mb{X}_{t}(\eps))} + \frac{\log(d_1 + d_2)}{\eta} } = \sqrt{2\En\nrm*{\sum_{t=1}^{n}\cM(\mb{X}_{t}(\eps))}\log(d_1 + d_2)}.
    \]
	To conclude, observe that for any sequence $(X_t)$ we have  
	  \[
	    \nrm*{\sum_{t=1}^{n}\cM(X_t)}_{\sigma} \leq{}  \max\crl*{\nrm*{\sum_{t=1}^{n}X_tX_t^{\trn}}_{\sigma}, \nrm*{\sum_{t=1}^{n}X_t^{\trn}X_t}_{\sigma}}.
	  \]
	  Indeed, $\sum_{t=1}^{n}\cM(X_t)  = \left(
	  \begin{array}{ll}
	  \sum_{t=1}^{n}X_tX_t^{\trn}& 0\\
	  0 & \sum_{t=1}^{n}X_t^{\trn}X_t
	  \end{array}
	  \right)$ and the spectral norm of a block-diagonal matrix is always obtained by the spectral norm of one of its blocks.

\end{proof}

\subsection{Proofs from \pref{sec:further}}
\label{app:square}

\begin{proof}[\textbf{Sketch of proofs for claims from \pref{sec:adagrad}}]

For the $\ls_{2}$ result we have
\begin{align*}
& \sum_{t=1}^n\loss(\pred_t,y_t) - \min_{\norm{w}_{2}\leq 1}\sum_{t=1}^n \loss(\inner{w,x_t},y_t) - 2L\sqrt{\sum_{t=1}^{n}\nrm*{x_t}^{2}_{2}} \\
& \leq{} \sup_{\norm{w}_{2}\leq 1} \crl*{\sum_{t=1}^n\partial\loss(\pred_t,y_t)(\yh_t, - \tri*{w, x_t}) } - 2L\sqrt{\sum_{t=1}^{n}\nrm*{x_t}^{2}_{2}} \\ 
& = \sum_{t=1}^n\partial\loss(\pred_t,y_t)\yh_t + \nrm*{\sum_{t=1}^{n}\partial\loss(\pred_t,y_t)x_t}_{2} - 2L\sqrt{\sum_{t=1}^{n}\nrm*{x_t}^{2}_{2}} \\
& \leq{} \sum_{t=1}^n\partial\loss(\pred_t,y_t)\yh_t + \burk_{\textrm{square}}\prn*{\sum_{t=1}^{n}\partial\loss(\pred_t,y_t)x_t, L\sqrt{\sum_{t=1}^{n}\nrm*{x_t}^{2}_{2}}}. 
\end{align*}
The path from here to a Burkholder function in the sense of \pref{lem:equivalence_burkholder} is clear given the three properties of $\burk_{\textrm{square}}$ stated in the main body.

For the $\ls_{\infty}$ result, 
\begin{align*}
& \sum_{t=1}^n\loss(\pred_t,y_t) - \min_{\norm{w}_{\infty}\leq 1}\sum_{t=1}^n \loss(\inner{w,x_t},y_t) - 2L\nrm*{\prn*{\sum_{t=1}^{n}x_{t}^{2}}^{1/2}}_{1} 
\end{align*}
can be upper bounded by
\begin{align*}
&\sup_{\norm{w}_{\infty}\leq 1} \crl*{\sum_{t=1}^n\partial\loss(\pred_t,y_t)(\yh_t, - \tri*{w, x_t}) } - 2L\nrm*{\prn*{\sum_{t=1}^{n}x_{t}^{2}}^{1/2}}_{1} \\ 
& = \sum_{t=1}^n\partial\loss(\pred_t,y_t)\yh_t + \nrm*{\sum_{t=1}^{n}\partial\loss(\pred_t,y_t)x_t}_{1} - 2L\nrm*{\prn*{\sum_{t=1}^{n}x_{t}^{2}}^{1/2}}_{1} \\
& \leq{} \sum_{t=1}^n\partial\loss(\pred_t,y_t)\yh_t + \sum_{i=1}^{d}\burk_{\textrm{square}}\prn*{\sum_{t=1}^{n}\partial\loss(\pred_t,y_t)x_t[i], L\sqrt{\sum_{t=1}^{n}\prn*{x_t[i]}^{2}_{2}}},
\end{align*}
where $x_{t}[i]$ refers to the $i$th coordinate of $x_t$. Once again, the three properties of $\burk_{\textrm{square}}$ directly lead to a valid Burkholder function $\burk$.
\end{proof}

\begin{proof}[\pfref{prop:square_loss_sufficient}]  
  Let $A_{n}=\rho\sum_{t=1}^{n}z_tz_t^{\trn}+\lambda{}I$ and $A_{0}=\lambda{}I$. Recall that $\Psi_{A}(w) = \frac{1}{2}\tri*{w,Aw}$. We begin by rewriting the desired regret bound as
  \[
    \cA(w;z_{1},\ldots,z_{n}) = \lambda\Phi((w,1)) + c\log\prn*{\det(A_n)/\det(A_{0})}
  \]
  for a constant $c>0$ to be determined. With this definition, we have
  \begin{align*}
    &\sup_{w\in\bbR^{d}}\crl*{
    \reg(w) - \cA(w;z_{1},\ldots,z_{n})
      } \\
    &=\sup_{w\in\bbR^{d}}\crl*{
      \sum_{t=1}^{n}\ls(\yh_t, y_t) - \sum_{t=1}^{n}\ls(\tri*{w,x_t}, y_t)  - \lambda\Phi((w,1))
      } - c\log\prn*{\det(A_n)/\det(A_{0})}
      \intertext{Using strong convexity of $\ls$:}
    &=\sup_{w\in\bbR^{d}}\crl*{
      \sum_{t=1}^{n}\partial{}\ls(\yh_t, y_t)(\yh_t - \tri*{w,x_t}) - \frac{\rho}{2}\prn*{\yh_t - \tri*{w,x_t}}^{2}  - \lambda\Phi((w,1))
      } - c\log\prn*{\det(A_n)/\det(A_{0})}\\
    &=\sup_{w\in\bbR^{d}}\crl*{
      \sum_{t=1}^{n}\partial{}\ls(\yh_t, y_t)(-\tri*{(w,1),z_t}) - \frac{\rho}{2}\prn*{\tri*{(w,1),z_t}}^{2}  - \lambda\Phi((w,1))
      } - c\log\prn*{\det(A_n)/\det(A_{0})}
      \intertext{We now move to an upper bound by allowing the final coordinate of $(w,1)$ to act as a free parameter.}
    &\leq{}\sup_{w\in\bbR^{d+1}}\crl*{
      \sum_{t=1}^{n}\partial{}\ls(\yh_t, y_t)\tri*{w,z_t} - \frac{\rho}{2}\tri*{w,z_t}^{2}  - \lambda\Phi(w)
      } - c\log\prn*{\det(A_n)/\det(A_{0})}
      \intertext{We can rewrite this as}
    &\leq{}\sup_{w\in\bbR^{d+1}}\crl*{
      \tri*{w,\sum_{t=1}^{n}\partial{}\ls(\yh_t, y_t)z_t} - \Psi_{\rho\Sigma_{n}}(w)  - \lambda\Phi(w)
      } - c\log\prn*{\det(A_n)/\det(A_{0})}\\
     &=\sup_{w\in\bbR^{d+1}}\crl*{
      \tri*{w,\sum_{t=1}^{n}\partial{}\ls(\yh_t, y_t)z_t} - \Psi_{A_{n}}(w)
      } - c\log\prn*{\det(A_n)/\det(A_{0})}\\
    &=\Psi_{A_n}^{\star}\prn*{
      \sum_{t=1}^{n}\partial{}\ls(\yh_t, y_t)z_t}
      - c\log\prn*{\det(A_n)/\det(A_{0})}.
  \end{align*}
  
This establishes that $\suff(x_t,\pred_t,\delta_t) = \left( \delta_{t}z_t, z_tz_t^{\trn} \right)\in\bbR^{d+1}\times{}\sym^{d+1}_{+}$ is a sufficient statistic. This is because we can write
\[
V(x, A) = \Psi^{\star}_{\rho{}A + \lambda{}I}\prn*{x} - c\log\prn*{\det(\rho{}A + \lambda{}I)/\det(A_0)}.
\]
and we just proved that
\[
\sup_{w\in\bbR^{d}}\crl*{\reg(w) - \cA(x_{1},\ldots,x_{n})
      }
      \leq{} V\prn*{
      \sum_{t=1}^{n}\suff(x_t,\pred_t,\delta_t)
      }.
\]

\end{proof}

\begin{proof}[\pfref{thm:square_loss_burkholder}]
Recall that we have defined $\burk(x, A) = V(x,A) =\Psi^{\star}_{A}\prn*{x} - c\log\prn*{\det(A)/\det(A_0)}$. We verify the properties from \pref{lem:equivalence_burkholder}. Property \proptwo{} is immediate, and for property \propone{} we have
\[
\burk(0) = \Psi^{\star}_{0 + \lambda{}I}(0) - c\log(\det(A_0)/\det(A_0)) = 0.
\]
We proceed to prove property \propthree{}. Fix $\tau=(\tau_1, \tau_2)\in\cT=\bbR^{d+1}\times{}\sym_{+}^{d+1}$ and a mean-zero distribution $p$ over $\brk*{-L, L}$. Then we have
\begin{align*}
\En_{\alpha\sim{}p}\burk(\tau + \suff(z, \alpha))
 &= \En_{\alpha\sim{}p}\brk*{
 \Psi^{\star}_{\rho(\tau_{2}+zz^{\trn})+\lambda{}I}(\tau_{1} + \alpha{}z) - c\log(\det(\rho(\tau_{2}+zz^{\trn})+\lambda{}I)/\det(A_0))
 } \\
 &= \En_{\alpha\sim{}p}\brk*{
 \Psi^{\star}_{\rho(\tau_{2}+zz^{\trn})+\lambda{}I}(\tau_{1} + \alpha{}z)} - c\log(\det(\rho(\tau_{2}+zz^{\trn})+\lambda{}I)/\det(A_0)).
\end{align*}
Let $A=\rho(\tau_{2}+zz^{\trn})+\lambda{}I$ and $B=\rho\tau_{2}+\lambda{}I$. Then since $\Psi^{\star}$ is a squared Euclidean norm and $\alpha$ is mean-zero:
\[
\En_{\alpha\sim{}p}\brk*{\Psi^{\star}_{A}(\tau_{1} + \alpha{}z)} \leq{} \Psi^{\star}_{A}(\tau_{1}) + \En_{\alpha\sim{}p}\brk*{\alpha^{2}\tri*{z, A^{-1}z}} \leq{} \Psi^{\star}_{A}(\tau_{1}) + L^{2}\brk*{\alpha^{2}\tri*{z, A^{-1}z}}.
\]
Also note that since $B\preceq{}A$, $\Psi^{\star}_{A}(\tau_1) \leq{} \Psi^{\star}_{B}(\tau_1)$.

To conclude, we first note that we just established
\begin{align*}
\En_{\alpha\sim{}p}\burk(\tau + \suff(z, \alpha)) &\leq{} \Psi_{B}^{\star}(\tau_1) + L^{2}\tri*{z, A^{-1}z} - c\log(\det(A)/\det(A_0)).
\intertext{Using a standard argument (e.g. from \cite{PLG}) and using that $A=B+\rho{}zz^{\trn}$:}
&\leq{} \Psi_{B}^{\star}(\tau_1) + \frac{L^{2}}{\rho}\log(\det(A)/\det(B)) - c\log(\det(A)/\det(A_0))
\intertext{For $c\geq{}L^{2}/\rho$, this is bounded by}
&\leq{} \Psi_{B}^{\star}(\tau_1) - c\log(\det(B)/\det(A_0)) \\
&= \burk(\tau).
\end{align*}

\end{proof}

\subsection{Proofs from \pref{sec:necessary}}

\begin{proof}[\pfref{prop:necessary}]
Recall that the regret inequality of interest is
\[
\sum_{t=1}^{n}\ls(\yh_t, y_t) - \inf_{f\in\cF}\sum_{t=1}^n \loss(f(x_t),y_t) - F\prn*{\sum_{t=1}^{n}\overline{\suff}(x_t)} \leq{} 0.
\]
As sketched in the \pref{sec:necessary}, \pref{lem:suff_to_martingale} shows that this is implied by
\begin{equation}
\sup_{\mb{z}}\En_{\eps}\brk*{V\prn*{\sum_{t=1}^{n}\suff(\mb{z}_t, \eps_t)}}\leq{}0,
\end{equation}
so the remainder of this proof will focus on the opposite direction. Suppose that $\ls(\yh,y)\ldef\abs*{\yh-y}$ is the absolute loss. We fix a Rademacher sequence $\eps_{1},\ldots,\eps_{n}$ and a tree $\x$ with $\x_{t}(\eps)=\x_{t}(\eps_{1},\ldots,\eps_{t-1})$. As a lower bound, consider a randomized adversary that plays $y_{t}=\eps_{t}$ and $x_{t} = \x_{t}(\eps)$. In this case the expected value of the regret inequality is
\[
\En_{\eps}\brk*{
\sum_{t=1}^{n}\ls(\yh_t, \eps_t) - \inf_{f\in\cF}\sum_{t=1}^n \loss(f(\x_t(\eps)),\eps_t) - F\prn*{\sum_{t=1}^{n}\overline{\suff}(\x_t(\eps))}
}.
\]
Observe that for any $\eps\in\pmo$ we have $\ls(\yh, \eps) = \abs*{1-\yh\eps}\geq{}1-\yh\eps$. Since the range of each $f\in\cF$ lies in $\brk*{-1,1}$, we have $\ls(f(x), \eps)=1-f(x)\eps$ exactly. The expected value of the regret inequality is therefore lower bounded by
\begin{align*}
&\En_{\eps}\brk*{
\sum_{t=1}^{n}(1-\yh_t\eps_t) - \inf_{f\in\cF}\sum_{t=1}^n (1-f(\x_t(\eps))\eps_t) - F\prn*{\sum_{t=1}^{n}\overline{\suff}(\x_t(\eps))}
} \\
&= \En_{\eps}\brk*{
- \inf_{f\in\cF}\sum_{t=1}^n (1-f(\x_t(\eps))\eps_t) - F\prn*{\sum_{t=1}^{n}\overline{\suff}(\x_t(\eps))}
} \\
&= \En_{\eps}\brk*{
\sup_{w\in\cW}\tri*{w, \sum_{t=1}^n\eps_{t}\x_{t}(\eps)} - F\prn*{\sum_{t=1}^{n}\overline{\suff}(\x_t(\eps))}
} \\
&= \En_{\eps}\brk*{
V\prn*{\sum_{t=1}^{n}\suff(\x_{t}(\eps), 0, \eps_t)
}}.
\end{align*}

For the final step, let $\wt{\y}$ be an arbitrary $\cY$-valued tree $\wt{\y}_{t}(\eps) = \wt{\y}_{t}(\eps_1,\ldots,\eps_{t-1})$. Using the explicit form for $V$, we have
\begin{align*}
\En_{\eps}\brk*{
V\prn*{\sum_{t=1}^{n}\suff(\x_{t}(\eps), \wt{\y}_{t}(\eps), \eps_t)
}}
&= \En_{\eps}\brk*{
\sum_{t=1}^{n}\eps_{t}\wt{\y}_{t}(\eps) + 
\sup_{w\in\cW}\tri*{w, \sum_{t=1}^n\eps_{t}\x_{t}(\eps)} - F\prn*{\sum_{t=1}^{n}\overline{\suff}(\x_t(\eps))}
} \\
&= \En_{\eps}\brk*{
0 + 
\sup_{w\in\cW}\tri*{w, \sum_{t=1}^n\eps_{t}\x_{t}(\eps)} - F\prn*{\sum_{t=1}^{n}\overline{\suff}(\x_t(\eps))}
} \\
& = \En_{\eps}\brk*{
V\prn*{\sum_{t=1}^{n}\suff(\x_{t}(\eps), 0, \eps_t)
}}.
\end{align*}

Since the argument above holds for any trees $\x$ and $\wt{\y}$, we conclude that the regret inequality implies that
\[
\sup_{\mb{z}}\En_{\eps}\brk*{V\prn*{\sum_{t=1}^{n}\suff(\mb{z}_t, \eps_t)}}\leq{}0.
\]
for all $\cX\times{}\cY$-valued trees.

\end{proof}

%%% Local Variables:
%%% mode: latex
%%% TeX-master: "paper"
%%% End:

\section{Burkholder Algorithm Implementation}
\label{app:efficient}

% !TEX root = paper.tex

\subsection{Generic Implementation}
In this section we assume that $\cY=\brk*{-B, B}$ for $B>0$ for simplicity. The only assumption we make on the form of $\burk$ is Lipschitzness and boundedness.

\begin{assumption}
The are constants $K_t$ and $H_t$ such that the mapping
\[
\yh\mapsto{}\burk\Big(\zeta_{t-1} + \suff(x_t,\pred,\partial \loss(\pred,y_t))\Big)
\]
is $K_t$-Lipschitz and bounded in magnitude by $H_t$ for any $y_t\in\cY$, $x_t\in\cX$, and $\zeta_{t-1}$ of the form $\zeta_{t}=\sum_{s=1}^{t}\suff(x_s, \pred_s, \partial(\pred_s, y_s))$.
\end{assumption}

Consider the following strategy:
\begin{itemize}
\item Fix precision $\veps_{1}>0$ and set $N=\ceil*{2B/\veps_{1}}$.
\item Define control points $z_{i} = -B + \veps_{1}\cdot{}i$ for $0\leq{}i\leq{}N$.
\item Let $\widehat{\mu}_{t}$ be a solution to the convex program
\begin{equation}
\label{eq:burkholder_approx}
\min_{\mu\in\Delta_{N}}\sup_{y\in\cY}\sum_{i=1}^{N}\mu_{i}\burk\Big(\zeta_{t-1} + \suff(x_t,z_{i},\partial \loss(z_{i},y))\Big)
\end{equation}
up to additive precision $\veps_{2}$.
\item Sample $\pred_{t}\sim{}\wh{\mu}_{t}$.
\end{itemize}

\begin{proposition}
\label{prop:burkholder_efficient}
Given a Burkholder function $\burk$, the strategy above guarantees
\[
\En\brk*{\sum_{t=1}^n \loss(\pred_t,y_t)} - \phi(x_1,y_1,\ldots,x_n,y_n) \leq{} \veps_{1}\sum_{t=1}^{n}K_t + \veps_{2}n.
\]
That is, the regret inequality \pref{eq:def_phi_regret} is obtained up to additive slack controlled by $\veps_{1}$ and $\veps_{2}$.
\end{proposition}
Before proving the theorem, let us discuss the computational prospects of implementing this strategy. First, suppose $K_t=K$ and $H_{t}=H\;\forall{}t\leq{}n$. To obtain the regret inequality up to constant error it suffices to take $\veps_{1}=1/Kn$ and $\veps_{2}=1/n$. In this case, we have $N=O(BKn)$. 

Now we must approximately solve \pref{eq:burkholder_approx}, which is a standard finite-dimensional convex non-smooth optimization problem. There are many possible solvers; we will choose Mirror Descent (e.g. \citep{nemirovskii1983problem,nesterov1998introductory,ben2001lectures}) for simplicity. Let $G(\mu)=\sup_{y\in\cY}\sum_{i=1}^{N}\mu_{i}\burk\Big(\zeta_{t-1} + \suff(x_t,z_{i},\partial \loss(z_{i},y))\Big)$. Our constraint set is $\ls_1$-bounded, and the boundedness assumption on $\burk$ implies that $G$ is $H$-Lipschitz with respect to the $\ls_{\infty}$ norm. In this case, Mirror Descent with the entropic regularizer (a.k.a. multiplicative weights) guarantees an $\veps$-approximate minimizer for $G(\mu)$ after $O\prn*{H\log(N)/\veps^{2}}$ update steps, each of which requires one evaluation of the subgradient of this function.

Evaluating the subgradient of $G(\mu)$ requires computing a supremum over $y\in\cY$. If $\burk\Big(\zeta_{t-1} + \suff(x_t,z_{i},\partial \loss(z_{i},y))\Big)$ is convex with respect to $y$, then the supremum is obtained in $\crl*{\pm{}B}$ and so can be checked in time $O(N)$. In this case, since each Mirror Descent update takes time $O(N)$, the total complexity of the algorithm is $O(BHKn^{3}\log(BKn))$.

If the supremum over $y\in\cY$ does not have a closed form, we can compute an approximate subgradient by taking a grid over the range $\brk*{-B,B}$ with spacing $\veps'$ and computing the $\argmax$ over this grid by brute force. If a $O(\veps)$-precision solution to the convex program is required, then it suffices to set $\veps'=\veps/K$ and use the approximate subgradients in the Mirror Descent scheme above. The approximate subgradient computation time is $O(KN/\veps)$ in this case, since we evaluate $\sum_{i=1}^{N}\mu_{i}\burk\Big(\zeta_{t-1} + \suff(x_t,z_{i},\partial \loss(z_{i},y))\Big)$ once per candiate $y$. The final time complexity is then $O(BHK^2n^{4}\log(BKn))$.

Lastly, we remark that if we replace Mirror Descent with Mirror Prox for saddle points \citep{nemirovski2004prox}, the dependence on $n$ in running time for the two cases above can be improved to $O(n^{2})$ and $O(n^{3})$ respectively.

The runtime can improved further if a regret bound of order $O(\sqrt{n})$ is sufficient, as this requires less precision.

\begin{proof}[\pfref{prop:burkholder_efficient}]

To begin, observe that since $\wh{\mu}_{t}$ is an approximate solution to \pref{eq:burkholder_approx}, it holds that
\[
\sup_{y\in\cY}\sum_{i=1}^{N}\wh{\mu}_{i}\burk\Big(\zeta_{t-1} + \suff(x_t,z_{i},\partial \loss(z_{i},y_t))\Big)
\leq{}
\inf_{\mu\in\Delta_{N}}\sup_{y\in\cY}\sum_{i=1}^{N}\mu_{i}\burk\Big(\zeta_{t-1} + \suff(x_t,z_{i},\partial \loss(z_{i},y_t))\Big)
+ \veps_{2}.
\]
The remainder of the proof will show that the right-hand-side above can be bounded as
\begin{align*}
\inf_{\mu\in\Delta_{N}}\sup_{y\in\cY}\sum_{i=1}^{N}\mu_{i}\burk\Big(\zeta_{t-1} + \suff(x_t,z_{i},\partial \loss(z_{i},y_t))\Big)
&\leq{} \inf_{q\in\Delta_{\cY}}\sup_{y\in\cY}\En_{\yh\sim{}q}\burk\Big(\zeta_{t-1} + \suff(x_t,\pred,\partial \loss(\pred,y))\Big)
+ K_{t}\veps_{1} \\
&\leq{} \burk(\zeta_{t-1})
+ K_{t}\veps_{1},
\end{align*}
where the second inequality follows from property \propthree{} of $\burk$ and was shown in the proof of \pref{lem:universal_algo}.

The first inequality can be seen as follows. Let $q\in\Delta_{\cY}$ and $y\in\cY$ be fixed. Let $F(z)\ldef{}\burk(\zeta_{t-1}, \suff(x_t, z, \partial\ls(z, y)))$. Since $q$ is a Borel probability measure and $F$ is continuous and bounded, $F$ is integrable with respect to $q$:
\[
\En_{\yh\sim{}q}\burk(\zeta_{t-1}, \suff(x_t, \yh, \partial\ls(\yh, y))) = \int_{\brk*{-B, B}}F(z)dq(z).
\]
Define $\cI_{1}=\brk*{z_0, z_1}$ and $\cI_{i}=(z_{i-1}, z_i]$ for $2\leq{}N$. Then $\crl*{\cI_i}$ form a partition of $\brk*{-B,B}$ and the integral can be approximated as
\begin{align*}
\int_{\brk*{-B, B}}F(z)dq(z) &= \sum_{i=1}^{N}\int_{\cI_{i}}F(z)dq(z) \\
&\geq{} \sum_{i=1}^{N}\int_{\cI_{i}}F(z_i)dq(z) - \sum_{i=1}^{N}\int_{\cI_{i}}\abs*{F(z_i)-F(z)}dq(z) \\
&= \sum_{i=1}^{N}q(\cI_i)F(z_i) - \sum_{i=1}^{N}\int_{\cI_{i}}\abs*{F(z_i)-F(z)}dq(z)\\
&\geq{} \sum_{i=1}^{N}q(\cI_i)F(z_i) - \sum_{i=1}^{N}\int_{\cI_{i}}K_{t}\veps_1dq(z)\\
&= \sum_{i=1}^{N}q(\cI_i)F(z_i) - K_{t}\veps_1\sum_{i=1}^{N}q(\cI_i) \\
&= \sum_{i=1}^{N}q(\cI_i)F(z_i) - K_{t}\veps_1.
\end{align*}
Since this holds for any $q\in\Delta_{\cY}$ and $y\in\cY$, we have
\begin{align*}
\inf_{q\in\Delta_{\cY}}\sup_{y\in\cY}\En_{\yh\sim{}q}\burk\Big(\zeta_{t-1} + \suff(x_t,\pred,\partial \loss(\pred,y))\Big)
&\geq{} \inf_{q\in\Delta_{\cY}}\sup_{y\in\cY}\sum_{i=1}^{n}q(\cI_i)\burk\Big(\zeta_{t-1} + \suff(x_t,z_i,\partial \loss(z_i,y))\Big) - K_t\veps_1 \\
&= \inf_{\mu\in\Delta_{N}}\sup_{y\in\cY}\sum_{i=1}^{n}\mu_i\burk\Big(\zeta_{t-1} + \suff(x_t,z_i,\partial \loss(z_i,y))\Big) - K_t\veps_1.
\end{align*}

\end{proof}

\subsection{Faster Implementation under Specific Structure}

In the remainder of this section of the appendix we show how to implement the Burkholder algorithm for certain special cases that enable admit especially simple strategies.

\begin{lemma}
  \label{lem:det_strat2}
  Suppose that the map
  \[
    \pred \mapsto \burk(\tau + \suff((x,\yh) , \partial(\pred, y)))
  \]
  is convex for all $y$. Then the strategy
\begin{align}
	\label{eq:det_strat2}
  \pred_{t}=\argmin_{\pred\in\cY}\sup_{y\in\cY} ~\burk\left(\sum_{j=1}^{t-1} \zeta_{t-1} + \suff(x_t,\yh,\partial \loss(\pred,y))\right)
\end{align}
achieves the value of the game in \pref{lem:universal_algo}.
\end{lemma}
\begin{proof}[\pfref{lem:det_strat2}]
  This follows by reduction to the general case:
  \begin{align*}
	  \inf_{\pred\in\cY}\sup_{y\in\cY} \burk\left(\zeta_{t-1} + \suff(x_t,\pred,\partial \loss(\pred,y))\right)
    &= \inf_{q\in\Delta_{\cY}}\sup_{y\in\cY} \burk\left(\zeta_{t-1} + \suff(x_t,\En_{\pred\sim{}q}\brk*{\pred},\partial \loss(\En_{\pred\sim{}q}\brk*{\pred},y))\right) \\
    &\leq \inf_{q\in\Delta_{\cY}}\sup_{y\in\cY} \En_{\yh\sim{}q}\burk\left(\zeta_{t-1} + \suff(x_t,\yh,\partial \loss(\yh,y))\right).
  \end{align*}
  The strategy in \pref{eq:det_strat2} is the minimax strategy for second expression above. The final expression is precisely the value of the Burkholder algorithm, which is controlled when $\burk$ is a Burkholder function via \pref{lem:universal_algo}.
\end{proof}
  
  \begin{lemma}
  \label{lem:det_strat3}
  Suppose that $\cY=\brk*{-B, B}$ for some $B>0$. Further suppose that we can write 
  \[
	\burk(\tau + \suff((x,\yh) , \dl)) = \yh\cdot\dl + F(\tau, x, \dl),
  \]
  where $\dl\mapsto{}F(\tau, x, \dl)$ is convex for all $\tau,x$.
Then the prediction strategy
  \begin{align}
	\label{eq:det_strat3}
  \pred_{t}= \mathrm{proj}_{\brk*{-B,B}}\prn*{-\frac{1}{L}\En_{\sigma\in\pmo}\brk*{
  \sigma{}F(\zeta_{t-1}, x_{t}, L\sigma)
  }},
  \end{align}
achieves the value of the game in \pref{lem:universal_algo}.
\end{lemma}

\begin{proof}[\pfref{lem:det_strat3}]
Let $\wt{y}_t$ denote the unprojected version of $\pred_t$:
\[
  \wt{y}_{t}= - \frac{1}{L}\En_{\sigma\in\pmo}\brk*{
  \sigma{}F(\zeta_{t-1}, x_{t}, L\sigma)
  }.
\]
We prove the lemma by inducting backwards. Let $t\in\brk*{n}$ be fixed. We first claim that
\begin{align*}
\sup_{y\in\cY} \burk\left(\zeta_{t-1} + \suff(x_t,\pred_t,\partial \loss(\pred_t,y))\right)
& = 
\sup_{y\in\cY}\brk*{\pred_t\cdot\partial \loss(\pred_t,y) + F(\zeta_{t-1}, x_t, \partial \loss(\pred_t,y))} \\ 
& \leq 
\sup_{y\in\cY}\brk*{\predt_t\cdot\partial \loss(\pred_t,y) + F(\zeta_{t-1}, x_t, \partial \loss(\pred_t,y))}.
\end{align*}
This holds by the assumption that $\argmin_{\yh\in\bbR}\ls(\yh, y)$ is obtained in $\brk*{-B, B}$ for any $y$. The assumption implies that for any $y$, $\partial\ls(\yh, y)\geq{}0$ for $\pred\geq{}B$ and $\partial\ls(\yh, y)\leq{}0$ for $\pred\leq{}-B$. If $\yh_{t}\neq{}\predt_t$, then either $\yh_{t}=B$ and $\predt_{t}>B$, so that $\partial\ls(\yh_{t}, y)\yh_{t}\leq{}\partial\ls(\yh_{t}, y)\predt_{t}$, or similarly $\yh_{t}=-B$ and $\predt_{t}<-B$, which also implies $\partial\ls(\yh_{t}, y)\yh_{t}\leq{}\partial\ls(\yh_{t}, y)\predt_{t}$.

Now, by the convexity assumption of the lemma, it holds that
\begin{align*}
\sup_{y\in\cY}\brk*{\predt_t\cdot\partial \loss(\pred_t,y) + F(\zeta_{t-1}, x_t, \partial \loss(\pred_t,y))}
&\leq{} \sup_{\delta\in\brk*{-L, L}}\brk*{\predt_t\cdot\delta + F(\zeta_{t-1}, x_t, \delta)} \\
&= \max_{\sigma\in\pmo}\brk*{\predt_t\cdot{}L\sigma + F(\zeta_{t-1}, x_t, L\sigma)}.
\end{align*}

The choice of $\predt_t$ guarantees that $\predt_t\cdot{}L\cdot(1) + F(\zeta_{t-1}, x_t, L\cdot(1)) = \predt_t\cdot{}L\cdot{}(-1) + F(\zeta_{t-1}, x_t, L\cdot(-1))$; this can be seen by rearranging this equality and solving for $\predt_t$. This means that we can take $\sigma=1$ to obtain the maximum in the expression above. Substituting in the value of $\predt_{t}$ then yields
\[
\max_{\sigma\in\pmo}\brk*{\predt_t\cdot{}L\sigma + F(\zeta_{t-1}, x_t, L\sigma)}
=\predt_t\cdot{}L\cdot(1) + F(\zeta_{t-1}, x_t, L\cdot(1)) = \En_{\sigma\in\pmo}\brk*{F(\zeta_{t-1}, x_t, \sigma{}L)}.
\]
Finally, we use property \propthreep{} of $\burk$ and the explicit form for $\burk$ assumed in the lemma statement to proceed back to time $t-1$:
\[
\En_{\sigma\in\pmo}\brk*{F(\zeta_{t-1}, x_{t}, \sigma{}L)} = \En_{\sigma\in\pmo}\brk*{\yh_{t}\sigma{}L+F(\zeta_{t-1}, x_{t}, \sigma{}L)}
= \En_{\sigma\in\pmo}\burk(\zeta_{t-1} + \suff((x_t,\yh_t) , \sigma{}L)) \leq{} \burk(\zeta_{t-1}).
\]
\end{proof}

\section{Algebra of \Bfun Functions}
\label{app:algebra}
% !TEX root = paper.tex

This appendix contains some additional structural results about Burkholder functions which may be useful for algorithm designers.

\begin{proposition}
\label{prop:algebra}
The following statements are true:
\begin{enumerate}
\item Given a \Bfun function $\burk$, if we define the $X_t = \burk(\sum_{j=1}^t \suff(z_j,\delta_j))$, then for any real-valued martingale difference sequence $\delta_t$s and predictable $z_t$s, $(X_t)_{t \ge 0}$ is a supermartingale with $\mathbb{E}[X_0] \le 0$.
\item Any convex combination of \Bfun functions is a \Bfun function.
\item The minimum of a family of \Bfun functions is a \Bfun function.
\item Suppose we have a finite set $A$ that indexes a family of functions $V_a:\cT\to\bbR$, each of which belongs to a sufficient statistic pair $(\suff, V_a)$ for some regret inequality of interest, and suppose each $V_a$ has a corresponding \Bfun function $\burk_a$.  Then the following probabilistic inequality is true:
$$
\mathbb{E}\left[\max_{a \in A} \left\{V_a\prn*{\sum_{t=1}^n \suff(z_t,\delta_t)} - \eta n C[a] \right\}\right]  \le \frac{1}{\eta}\log |A|,
$$
where $C[a] = \sup_{\tau, z ,\alpha} (\burk_a(\tau + \suff(z,\alpha)) - \burk_a(\tau))^2$. Note that $C \in \mathbb{R}^A$ may be thought as a sufficient statistic, though it is fixed and does not depend on instances.
Furthermore, a \Bfun function $\burk: \T \times \mathbb{R}^A \to \reals$ that certifies this inequality is:
\begin{equation}
\label{eq:burkholder_meta}
\burk(\tau, \gamma) = \frac{1}{\eta} \log\left(\sum_{a \in A} \exp\left(\eta \burk_a(\tau) - \eta^2 \gamma[a] \right)\right)  - \frac{\log|A|}{\eta}
\end{equation}
\end{enumerate}
\end{proposition}
\begin{proof}[\pfref{prop:algebra}]
The first statement follows from property \propthree{} of the Burkholder function $\burk$, which immediately implies that it is a supermartingale. The second statement is trivial. To prove the third statement it suffices to verify property \propthree, which holds due to concavity of the minimum. 

We now prove the fourth statement. Given a family of \Bfun functions $\crl*{\burk_a}_{a\in{}A}$, define a new \Bfun function $\burk: \T \times \mathbb{R}^A \to\reals$ as:
$$
\burk(\tau, \gamma) = \frac{1}{\eta} \log\left(\sum_{a \in A} \exp\left(\eta \burk_a(\tau) - \eta^2 \gamma[a] \right)\right)  - \frac{\log|A|}{\eta}.
$$
whose sufficient statistics are the original sufficient statistic of the family of $V_a$s along with an additional $|A|$-dimensional real vector, for which one coordinate per $a \in A$ will be used to represent $C[a] = \sup_{\tau, z ,\alpha} (\burk_a(\tau + \suff(z,\alpha)) - \burk_a(\tau))^2$ (note that this is a vacuous statistic as it is constant for each instance). Property \propthree{} for $\burk$ holds as follows:
\begin{align*}
 \En_\alpha &\burk\left((\tau,\gamma) + (\suff(z,\alpha) , C) \right) \\
 & = \frac{1}{\eta} \En_\alpha \log\left(\sum_{a \in A} \exp\left(\eta \burk_a(\tau + \suff(z,\alpha)) - \eta^2 \gamma[a] - \eta^2 C[a]\right)\right) - \frac{\log|A|}{\eta} \\
 & \le  \frac{1}{\eta}  \log\left(\sum_{a \in A} \En_\alpha \exp\left(\eta \burk_a(\tau + \suff(z,\alpha))  - \eta^2 \gamma[a] - \eta^2 C[a] \right)\right) - \frac{\log|A|}{\eta} \\
 & =  \frac{1}{\eta}  \log\left(\sum_{a \in A} \En_\alpha \exp\left(\eta \left(\burk_a(\tau + \suff(z,\alpha)) - \burk_a(\tau) \right) + \eta\burk_a(\tau)  - \eta^2 \gamma[a] - \eta^2 C[a]\right)\right) - \frac{\log|A|}{\eta}.
\end{align*}
Now note that by property \propthree{} of the \Bfun functions $\crl*{\burk_a}_{a\in{}A}$, the random variable $X_a = \left(\burk_a(\tau + \suff(z,\alpha)) - \burk_a(\tau) \right)$ is such that $\En_{\alpha}[X_a] \le 0$. Further from our assumption we have that $|X_a|^2 \le C[a]$. Hence, the standard mgf bound implies $\mathbb{E}_{\alpha}[\exp(\eta X_a)] \le \exp(\eta^2 C[a]/2)$.
\begin{align*}
 & \le  \frac{1}{\eta}  \log\left(\sum_{a \in A} \exp\left(\eta\burk_a(\tau)  + \frac{\eta^2}{2} C[a]  - \eta^2 \gamma[a] - \eta^2 C[a] \right)\right) - \frac{\log|A|}{\eta} \\
 & \le \frac{1}{\eta}  \log\left(\sum_{a \in A} \exp\left(\eta\burk_a(\tau)   - \eta^2 \gamma[a] \right)\right) - \frac{\log|A|}{\eta}.
\end{align*}
For property \propone{} it can be seen immediately that $\burk(0) \le 0$. Property \proptwo{} holds via
\begin{align*}
 \burk(\tau,\gamma) &= \frac{1}{\eta} \log\left(\sum_{a \in A} \exp\left(\eta \burk_a(\tau) - \eta^2 \gamma[a] \right)\right)  - \frac{\log|A|}{\eta}\\
& \ge \max_{a \in A}\left\{ \burk_a(\tau) - \eta \gamma[a]\right\} - \frac{\log|A|}{\eta} ~~~~~~ \textrm{(softmax upper bounds max)}\\
& \ge \max_{a \in A}\left\{ V_a(\tau) - \eta \gamma[a]\right\} - \frac{\log|A|}{\eta}.
\end{align*}
\end{proof}
We remark that one uses non-additive sufficient statistics as discussed in \pref{sec:discussion}, then one can make the bound implied by the Burkholder function $\burk$ above more data-dependent by replacing $C[a]$ with $\sup_{\delta} \left(\burk_a(\tau + \suff(z,\delta)) - \burk_a(\tau) \right)^2$ for each $a$.

\end{document}